\documentclass{article}

\usepackage{microtype}
\usepackage{graphicx}
\usepackage{subcaption}
\usepackage{booktabs} % for professional tables

\usepackage{hyperref}

\usepackage[preprint]{icml2026}

\usepackage{amsmath}
\usepackage{amssymb}
\usepackage{mathtools}
\usepackage{amsthm}

\usepackage{pifont}
\newcommand{\poly}{\operatorname{poly}}

\usepackage[capitalize,noabbrev]{cleveref}

\theoremstyle{plain}
\newtheorem{theorem}{Theorem}[section]

\newtheorem{lemma}[theorem]{Lemma}

\theoremstyle{definition}
\newtheorem{definition}[theorem]{Definition}
\newtheorem{assumption}[theorem]{Assumption}
\theoremstyle{remark}
\newtheorem{remark}[theorem]{Remark}

\usepackage[textsize=tiny]{todonotes}

\icmltitlerunning{Submission and Formatting Instructions for ICML 2026}

\begin{document}

\twocolumn[
  \icmltitle{Can Test-time Computation Mitigate Reproduction Bias in \\Neural Symbolic Regression?}

  % It is OKAY to include author information, even for blind submissions: the
  % style file will automatically remove it for you unless you've provided
  % the [accepted] option to the icml2026 package.

  % List of affiliations: The first argument should be a (short) identifier you
  % will use later to specify author affiliations Academic affiliations
  % should list Department, University, City, Region, Country Industry
  % affiliations should list Company, City, Region, Country

  % You can specify symbols, otherwise they are numbered in order. Ideally, you
  % should not use this facility. Affiliations will be numbered in order of
  % appearance and this is the preferred way.
  \icmlsetsymbol{equal}{*}

  \begin{icmlauthorlist}
    \icmlauthor{Shun Sato}{yyy}
    \icmlauthor{Issei Sato}{yyy}
  \end{icmlauthorlist}

  \icmlaffiliation{yyy}{Department of Computer Science, 
Graduate School of Information Science and Technology, The University of Tokyo, Tokyo, Japan}

  \icmlcorrespondingauthor{Shun Sato}{sato-shun471@g.ecc.u-tokyo.ac.jp}

  % You may provide any keywords that you find helpful for describing your
  % paper; these are used to populate the "keywords" metadata in the PDF but
  % will not be shown in the document
  \icmlkeywords{Machine Learning, ICML}

  \vskip 0.3in
]

% this must go after the closing bracket ] following \twocolumn[ ...

% This command actually creates the footnote in the first column listing the
% affiliations and the copyright notice. The command takes one argument, which
% is text to display at the start of the footnote. The \icmlEqualContribution
% command is standard text for equal contribution. Remove it (just {}) if you
% do not need this facility.

% Use ONE of the following lines. DO NOT remove the command.
% If you have no special notice, KEEP empty braces:
\printAffiliationsAndNotice{}  % no special notice (required even if empty)
% Or, if applicable, use the standard equal contribution text:
% \printAffiliationsAndNotice{\icmlEqualContribution}

\begin{abstract}

Mathematical expressions play a central role in scientific discovery.
Symbolic regression aims to automatically discover such expressions from given numerical data. Recently, Neural symbolic regression (NSR) methods that involve Transformers pre-trained on synthetic datasets have gained attention for their fast inference, but they often perform poorly, especially with many input variables. 
In this study, we analyze NSR from both theoretical and empirical perspectives and show that (1) naive token-by-token generation is ill-suited for NSR, as Transformers cannot compositionally generate tokens while validating numerical consistency, and (2) the search space of NSR methods is greatly restricted due to reproduction bias, where the majority of generated expressions are merely copied from the training data.
We further examine whether tailored test-time strategies can reduce reproduction bias and show that providing additional information at test time effectively mitigates it. These findings contribute to a deeper understanding of the limitation of NSR approaches and provide guidance for designing more robust and generalizable methods. Code is available at \url{https://github.com/Shun-0922/Mem-Bias-NSR}.
\end{abstract}

\section{Introduction}
Discovering underlying equations from collected experimental data is a crucial process in many fields of scientific research.
Symbolic regression is a branch of regression analysis that seeks to automatically identify the underlying mathematical expressions. 
In contrast to methods that model data without explicit mathematical expressions,
symbolic regression offers advantages in terms of interpretability and generalizability. This is because the outputs of symbolic regression are usually compact, human-readable equations, making them less susceptible to overfitting. However, symbolic regression is a challenging task due to its vast search space; the number of possible mathematical expressions grows exponentially with expression length or the number of input variables. Applications for symbolic regression span various fields of scientific research such as physics \citep{tenachi2023deep}, materials science \citep{wang2019symbolic}, and weather forecasting \citep{abdellaoui2021symbolic}.

Various methods for symbolic regression have been proposed in recent years. Traditionally, approaches based on genetic programming (GP) \citep{koza1994genetic} have been used to solve symbolic regression. These methods tend to be computationally expensive because they generate each expression entirely from scratch. To mitigate this inefficiency, a research direction called neural symbolic regression (NSR) has emerged. 
NSR methods leverage encoder-decoder Transformer architectures \cite{vaswani2017attention} pre-trained on large-scale synthetic datasets \cite{biggio2021neural, valipour2021symbolicgpt}. NSR methods generate expressions similar to natural language processing tasks, where expressions are generated token-by-token in an auto-regressive manner. 
Since a single forward pass through the Transformer suffices to output a mathematical token (e.g., $x_1, \sin, +$), NSR models can generate solutions far more quickly than GP-based approaches. However, NSR methods often fall short in terms of numerical accuracy, with particularly poor performance when the number of input variables is large \citep{kamienny2022end, bendinelli2023controllable}. This study aims to uncover the underlying cause of this drawback and explore methods to alleviate it.

Our analysis begins by examining the mechanisms a Transformer relies on to select the next token during expression generation. Specifically, we theoretically analyze the limitations Transformers face when generating mathematical expressions.
A natural and ideal token-generation strategy would be to select, in a compositional manner at each step, the token that is most appropriate in light of the numerical data.
However, by using circuit complexity theory, we show that Transformers fail to generate expressions in such ways; they cannot compositionally generate mathematical expressions while taking numerical data into account. In other words, while Transformers generate tokens sequentially, they are not able to carry out meaningful processing at every step of token generation, suggesting that naive token-by-token generation is not particularly suitable for NSR.
To illustrate this limitation, consider a situation where a Transformer has generated an expression up to $x_1^2 + \sin(x_2) +$. Our analysis implies that Transformers are unable to internally compute which leaf token (e.g., $x_1, x_2, x_3, \ldots$) would be the most appropriate given the input numerical data. 
The result indicates that in practice, Transformers generate expressions by some alternative, more coarse-grained mechanism instead of generating tokens in a compositional manner.

We next investigated how NSR methods actually generate expressions under empirical conditions.
We hypothesize that, in NSR methods, naively using a Transformer for inference leads to \textbf{reproduction bias}, meaning that models struggle to generate novel expressions not seen during training and instead tend to generate expressions copied from the training data. Given that the expressions in the training data typically represent only a small subset of the full space of possible expressions, our hypothesis implies that standard NSR methods operate within a significantly constrained search space. We investigated this hypothesis in NSR methods such as NeSymReS \citep{biggio2021neural}, a pioneering work in NSR models. We found that the majority of expressions generated by Transformers are expressions that were included in the training dataset, which supports our hypothesis of reproduction bias. Prior work has highlighted NSR methods’ limited generalizability with respect to the range of numerical data—e.g., models trained on data whose input variable $x$ lies in the interval $[-1,1]$ often fail when evaluated on inputs from the wider interval $[-2,2]$ \citep{li2024generative, shojaee2023transformer}. However, the reproduction bias that we identify is orthogonal to this phenomenon and represents an even more fundamental limitation: NSR models often fail to generalize even within their training domain. This work is the first to show that standard NSR models primarily copy training expressions instead of composing familiar components into genuinely novel formulas.

Towards the end of this paper, we explore methodologies to mitigate the reproduction bias of standard NSR models and improve numerical accuracy. Since the theoretical analysis suggested that naive auto-regressive generation may be inappropriate for NSR, we focus particularly on test-time strategies and investigate how they affect reproduction bias and numerical accuracy. We compare three strategies: decoding with a large beam size, decoding using MCTS, and providing verification feedback at the subtree level. The last strategy is a new method that we propose, which we refer to as neural symbolic regression guided by verified subtrees (NSR-gvs). We found that providing new information to the model during test-time leads to generating expressions beyond the training dataset. However, we also identify cases where reproduction bias was mitigated but numerical accuracy decreased, as well as cases where numerical performance improved despite little alleviation in reproduction bias.

The contributions of our work are summarized as follows:
\begin{itemize}
    \item We formally show that Transformers lack the ability to compositionally generate expressions while accounting for numerical data.
    \item We empirically demonstrate, under various settings, that naively applying a Transformer to symbolic regression leads to reproduction bias. 
    \item We compare varying test-time computing strategies and analyze how such strategies affect reproduction bias and numerical accuracy.
\end{itemize}

\section{Related Work}
\label{Sec:related_work}

\begin{table*}
  \caption{Comparison between our work and other major NSR studies}
  \label{Tab:related_work}
  \centering
    \begin{tabular}{ccccc}
    \toprule
        NSR Methods & Automatic Training & Direct Constant & Information Added & Assessing \\
        ~ & Data Generation & Prediction & During Test-time & Reproduction Bias \\
        \midrule
        \citet{biggio2021neural} & \ding{51} & - & - & -\\
        \citet{kamienny2022end} & \ding{51} & \ding{51} & - & - \\
        \citet{shojaee2023transformer} & \ding{51} & \ding{51} & MCTS Feedback & - \\
        \citet{li2024generative} & - & - & Historical Context & - \\
        \citet{bendinelli2023controllable} & \ding{51} & - & Prior Knowledge & -\\
        Ours (NSR-gvs)& \ding{51} & - & Verified Subtrees & \ding{51} \\
    \bottomrule
    \end{tabular}
\end{table*}

Several approaches to symbolic regression exist, such as GP, brute-force algorithms, reinforcement learning, and NSR. Since our study focuses on analyzing and improving NSR methods, we mainly describe NSR in detail in this section, and explain other symbolic regression methods in Appendix \ref{Sec:additional_related_work}.

Traditional symbolic regression methods such as GP generate each equation from scratch, resulting in long inference times, with equation generation potentially taking hours. In order to achieve a shorter inference time, studies such as NeSymReS \citep{biggio2021neural} and SymbolicGPT \citep{valipour2021symbolicgpt} carried out large scale pre-training of Transformers. In these studies, an artificial dataset consisting of millions of randomly generated equations was used for training. These methods, often categorized as NSR, can generate an expression in just a few seconds, significantly reducing inference time compared with other approaches.

In recent years, several studies, summarized in Table \ref{Tab:related_work}, have focused on improving NSR methods. Studies such as \citep{kamienny2022end} and \citep{vastl2024symformer} proposed an end-to-end approach using a Transformer model to directly predict full mathematical expressions including constants, whereas previous methods followed a two-step procedure where constant fitting had to be done separately. \citet{lalande2023transformer} analyzed several different architectures to find a suitable encoder architecture for NSR. \citet{shojaee2023transformer} focused on improving the decoding strategy for NSR, incorporating the MCTS algorithm during the generation of expressions.
In their study, \citet{li2024generative} trained a Transformer model to imitate the process of improving mathematical formulas, as performed in the reinforcement learning-based approach proposed by \citet{mundhenk2021symbolic}. \citet{bendinelli2023controllable} proposed a model called NSRwH that enables incorporating prior knowledge, which is often available during application in scientific research. For example, scientists may anticipate symmetries between variables or expect certain partial expressions to appear in the mathematical laws governing the data. NSRwH adds a dedicated encoder that processes such prior knowledge, allowing the model to generate expressions that are consistent with both the numerical data and the prior knowledge provided.

More recently, there has been growing research on methods that iteratively refine mathematical expressions, most of which rely on large language models (LLMs) rather than pre-trained Transformer models. These methods are similar to NSR-gvs, one of the test-time computation approaches considered in this paper, in that they iteratively improve their outputs by incorporating feedback from the generated expressions. In \citep{merler2024context} and \citep{sharlin2024context}, the authors introduce approaches in which a base equation structure is generated using LLMs, and the equation is subsequently improved iteratively by receiving feedback from external numerical solvers. While \citet{shojaee2023transformer} follow a similar methodology, it incorporates supplementary descriptions regarding the variables in the prompt, facilitating more effective use of the LLM’s scientific knowledge. \citet{grayeli2024symbolic} proposed a method that uses an LLM to identify “concepts” representing features of high-performing expressions and leverages them to further evolve a set of equations. \citet{zhang2025rag} introduces an iterative algorithm that replaces features of suboptimal expressions at each step while incorporating relevant expressions as needed.
Pre-training-free methods described above may have the potential to address some of the limitations of conventional NSR approaches. On the other hand, using a pre-trained Transformer, as opposed to an LLM, offers certain advantages similar to those of small language models (SLMs), such as keeping the model size manageable and enhancing domain specificity through careful design of the training data. For these reasons, we believe that conducting a deeper analysis of pre-trained Transformer-based NSR and exploring ways to improve it remains a valuable research direction.

\section{Problem Formulation}\label{Sec:preliminary}

We formalize NSR as the problem of learning a \emph{parameterized symbolic regressor}
$S_{\boldsymbol\theta}$ that maps a numerical dataset
$\mathcal{D}$ to a symbolic expression
$\hat e = S_{\boldsymbol\theta}(\mathcal{D})$.
The learning algorithm is formulated as minimizing a loss that measures how well $\hat e$ matches the ground-truth expression $e^{\ast}$ underlying $\mathcal{D}$.
In this section, we specify how synthetic training pairs
$(e^{\ast},\mathcal{D})$ are generated in NeSymReS \citep{biggio2021neural}, since it is the foundational work underlying our research.

%-------------------------------------------------
\subsection{Synthetic Expression Distribution}\label{Ssec:expr_dist}

We first sample a random
binary–unary tree
whose internal nodes are operators and whose leaves are
variables.

Let $\mathcal V=\{x_1,\dots,x_d\}$ be a finite set of variables,
$\mathcal O_{\mathrm{bin}}$ the binary operators (e.g., $\{+,\, -,\, \times,\, \div\}$),  and $\mathcal O_{\mathrm{un}}$ the unary operators (e.g., $\{\sin,\,\cos,\,\log,\,\exp \}$).
Denote by $\mathcal{C}=[c_{\min},c_{\max}]\subset\mathbb{R}$ the interval from which numeric constants are drawn. The complete vocabulary for the expression is $\Sigma=\mathcal{V}\cup\mathcal{O}_{\mathrm{bin}}\cup\mathcal{O}_{\mathrm{un}}\cup\mathcal{C}$.

Let $\mathcal{E}$ be an expression space and \(p_{\mathcal{E}}\) be the generator of symbolic expressions employed in NeSymReS \citep{biggio2021neural}. We also denote by $p_{\mathrm{Tree}}$ the generator of unary-binary trees introduced by \citet{lample2019deep}.
We write $e^{\ast}\sim p_{\mathcal{E}}$ for the following procedure.
\begin{enumerate}
\item Draw a random binary-unary tree $T \sim p_{\mathrm{Tree}}$.
\item Assign internal nodes independently and uniformly from
$\mathcal {O}_{\mathrm{bin}}\cup\mathcal{ O}_{\mathrm{un}}$, and leaves uniformly from the variable set $\mathcal V$,
resulting in a \emph{template} expression $e_{\mathrm{templ}}$.
\item For each unary operator $u$, sample a constant $c_\mathrm{mul}$ from distribution $\mathcal{D}_\mathrm{mul}$ and replace $u$ with $c_\mathrm{mul} u$; otherwise keep the unary operator as is.
\item For each variable $x$, sample a constant $c_\mathrm{mul}$ from distribution $\mathcal{D}_\mathrm{mul}$ and a constant $c_\mathrm{add}$ from distribution $\mathcal{D}_\mathrm{add}$ and replace $x$ with $c_\mathrm{mul} x + c_\mathrm{add}$; otherwise keep the variable as is.
\item The resulting expression is the final
      $e^{\ast}\in\mathcal E$.
\end{enumerate}

\subsection{Synthetic Dataset Generation}
\label{Ssec:data_dist}
Given an expression $e^{\ast}$, we construct the dataset
\[
\mathcal{D}
  \;=\;
  \{\,(\mathbf x_i,\,y_i)\,\}_{i=1}^{n},
  \quad
   x_{ij} \;\sim\; \mathcal{U}([x_{\min,j},x_{\max,j}])
\]
\[
  \text{for}
  \quad
   j=1,2,\ldots,d,
  \quad
  y_i \;=\; e^{\ast}(\mathbf x_i). 
\]
Where $\{[x_{min,j},x_{max,j}]\}_{j=1}^d$ denotes the intervals for each independent variable. The joint distribution of training pairs is therefore
\(
  (e^{\ast},\mathcal D)
  \sim
p_{\mathcal{E}}\times\mathcal G,
\)
where $\mathcal G$ denotes the above stochastic data generation process.

We now denote by $\Gamma =\mathcal{V}\cup\mathcal{O}_{\mathrm{bin}}\cup\mathcal{O}_{\mathrm{un}}\cup\{C,\textsc{End}\}$ the vocabulary for token sequences, where $C$ is a placeholder token to represent constants, and \textsc{End} denotes the explicit end-of-sequence marker. The vocabulary $\Gamma$ is slightly different from $\Sigma$ since continuous numeric constants cannot be represented with a finite number of tokens.

Let 
\(
  \operatorname{seq}:\mathcal E\!\longrightarrow\!\Gamma^{\ast}
\) 
be a \emph{serialization map} that converts any symbolic
expression into its unique prefix token representation.
For a ground–truth expression \(e^{\ast}\in\mathcal E\) we set
\[
  \mathbf s^{\ast}
  \;=\;
  \operatorname{seq}\bigl(e^{\ast}\bigr)
  \;=\;
  (s^{\ast}_{1},\ldots,s^{\ast}_{L}),
  \qquad
  L:=|\mathbf s^{\ast}|.
\]

\noindent
The predictive distribution 
\(
  q_{\boldsymbol\theta}
  (\,\cdot\,\mid \mathbf s_{<j},\mathcal D)
\)
is realized by an \emph{encoder–decoder Transformer}
parametrized by \(\boldsymbol\theta\).
Conditioned on the dataset~\(\mathcal D\) (encoded by the encoder) and
the previously emitted prefix \(\mathbf s_{<j}\),
the decoder outputs a probability over the next token
\(s_{j}\in\Gamma\).

The \emph{token-level loss} for a single training pair
\((e^{\ast},\mathcal D)\) is then
\begin{align}
L_{\mathrm{tok}}\!\bigl(e^{\ast};\boldsymbol\theta\bigr)
=
-\sum_{j=1}^{L}\log q_{\boldsymbol{\theta}}(      s^{\ast}_{j}\mid\mathbf{s}^{\ast}_{<j},\,\mathcal{D}).
\end{align}

Note that, in practice, the training dataset is the collection of $e_{\mathrm{templ}}$, and both $e^\ast$ and $\mathcal{D}$ are generated dynamically during training. Further details concerning the work of NeSymReS \citep{biggio2021neural} are described in Appendix \ref{Sec:detail_nesymres}.

\begin{figure*}[!t]
  \begin{center}
  \centerline{\includegraphics[width=1.6\columnwidth]{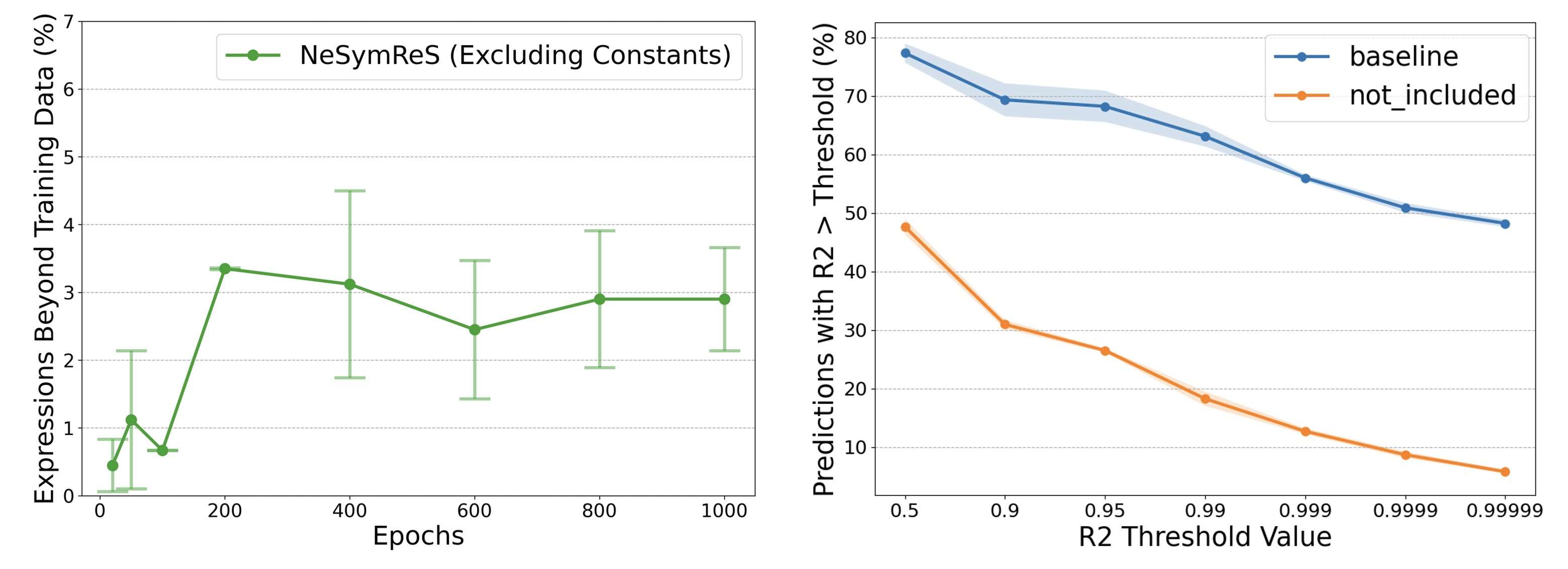}}
  \caption{(Left) Percentage of expressions beyond the training dataset generated by NeSymReS on the \texttt{not\_included} dataset. Throughout the training procedure, NeSymReS can hardly generate expressions that are not included in the training data, indicating strong reproduction bias.
  (Right) NeSymReS exhibits strong fitting performance on the \texttt{baseline} dataset but performs poorly on expressions from the \texttt{not\_included} dataset, whose tree structures are absent from the training data. The result indicates the severe effect of reproduction bias on numerical accuracy.}
  \label{Fig:baseline_vs_not_included}
  \end{center}
\end{figure*}

\section{Theoretical Analysis on Expression Generation Ability of Transformers}
\label{Sec:theoretical_analysis}

Transformer-based symbolic regression tends to suffer from low performance, particularly when the number of input variables is large.
In this section, we explore the theoretical basis of this limitation.
Ideally, Transformers should be able to generate tokens in a compositional manner, while maximizing the probability for the final expression to fit the numerical data.
However, we show that Transformers inherently lack the capacity to generate expressions compositionally while accounting for their numerical characteristics.
We introduce a simplified version of the symbolic regression task and show that Transformers are not expressive enough to solve the task.

We define the \emph{last-token prediction problem} as the task of predicting the most suitable final token in an otherwise complete mathematical expression. Although predicting the entire optimal expression is NP-hard \cite{virgolin2022symbolic}, this task is much easier since the search space is limited to several leaf tokens. We present a formal definition of this task in the following.

We first introduce $\mathrm{expr}: \Gamma^\ast \times \mathbb{R}^{(d+1)*n} \!\longrightarrow\! \mathcal{E}$, a function that maps a token sequence $\mathbf{s}$ to the most appropriate expression $e_{\mathbf{s}}$ that can be represented by $\mathbf{s}$, taking into account the numerical data $\mathcal{D}$. Since the token sequence may contain the placeholder token $C$ representing constants, the mapping is tasked with identifying the optimal values for these constants and transforming the sequence into a corresponding expression tree.

\begin{definition}[Last-token prediction problem]
Given numerical data $\mathcal{D}$ of $n$ features-value pairs $(\mathbf x_i,y_i) \in \mathbb{R}^{d} \times \mathbb{R}$, a metric $\mathcal{L}: \mathbb{R}^n \times \mathbb{R}^n \to \mathbb{R}$, and an incomplete token sequence $\widetilde{\mathbf{s}}$ that forms a prefix representation of an expression when terminated with a leaf token, the last-token prediction problem asks to find a leaf token $u^\ast$ such that:
\[
u^\ast = \underset{u \in \Gamma}{\operatorname{arg min}}~\mathcal{L}(\mathbf{y},e_{(\widetilde{\mathbf{s}},u)}(\mathbf{x})),
\]
where $e_{(\widetilde{\mathbf{s}},u)} = \mathrm{expr}((\widetilde{\mathbf{s}},u),\mathcal{D})$ with $(\widetilde{\mathbf{s}},u)$ representing the concatenation of sequence $\widetilde{\mathbf{s}}$ and token $u$. When the length of $\widetilde{\mathbf{s}} = m$, we denote this problem as $\operatorname{LastTokenPrediction}(m)$.
\end{definition}

For the analysis, we assume a bounded-depth log-precision Transformer as in \citep{feng2023towards, merrill2023parallelism, merrill2023logic, strobl2023average}, a realistic setting where the intermediate computation values of the Transformer are limited to $O (\log k)$ bit precision, with $k$ denoting the maximal length of the input sequence.
We now present the theoretical result stating that Transformers with bounded size cannot solve the last-token prediction problem.

\begin{theorem}
\label{Thm:theorem1}
Assume $\mathsf{TC}^0 \neq \mathsf{NC}^1$. For any integer $D$ and any polynomial $Q$, there exists a problem size $m$ such that  no log-precision Transformer defined in Section \ref{Ssec:theory_preliminary} with depth $D$ and hidden dimension $d \leq Q(m)$ can solve $\operatorname{LastTokenPrediction}(m)$.
\end{theorem}

We show the above theorem by leveraging circuit complexity theory. Specifically, $\mathsf{TC}^0$ and $\mathsf{NC}^1$ are types of circuit complexity classes, and it is generally conjectured that $\mathsf{TC}^0 \subsetneq \mathsf{NC}^1$.
Prior work \citep{merrill2023parallelism} shows that log-precision Transformers can be simulated with $\mathsf{TC}^0$ circuits. We provide a proof for the above theorem by showing that the complexity of the last-token prediction problem is lower bounded by $\mathsf{NC}^1$.
Detailed specifications of the problem setting and proof of the theorem are provided in Appendix \ref{Sec:theory}. 

Although the final token of a mathematical expression is arguably the easiest to predict among its components, the above theorem shows that even this seemingly simple task presents substantial difficulties for Transformer models. 
In general, we find that when Transformers generate tokens via a single forward pass, they lack the capacity to account for numerical consistency, making it questionable whether auto-regressive generation is an appropriate paradigm for symbolic expression generation.

\section{Exploring Reproduction Bias in NSR}\label{Sec:reproduction_bias}

When generating expressions in an auto-regressive manner, a seemingly appropriate strategy would be to compositionally produce the next token that maximizes accuracy, conditioned on both the previously generated partial expression and the numerical data. However, our theoretical analysis from the last section showed that Transformers lack the ability to do so, meaning that they rely on more coarse-grained strategies. This brings us to the following question: \textit{How, in practice, does a transformer generate expressions during inference?} In this section, we empirically analyze how expressions are actually generated by NSR models.
We demonstrate that NSR models primarily rely on reproduction—that is, they tend to generate expressions by directly copying those seen in the training data.

\subsection{Reproduction Bias in Simplified Setting} \label{Ssec:experiment_simple}

We first tested how expressions are generated in NeSymReS, which is the method that we mainly focus on in this study. We examined whether the expressions generated by NeSymReS are merely copies from the training data or newly constructed formulas generated compositionally by the model.

\paragraph{Definition of reproduction bias.}
We first describe how the novelty of expressions are measured in this section.
We introduce 
\(
\operatorname{seq}^{-1} : \Gamma^{\ast}\!\longrightarrow\!\mathcal E
\),
which is the inverse function of \(\operatorname{seq}\) defined in Section \ref{Ssec:data_dist}, and 
\(
\operatorname{strip} : \mathcal E\!\longrightarrow\!\mathcal E
\), 
which is a function that removes all constants from the input expression.
We also define $E_{\operatorname{templ}}$ as the set of all $e_{\operatorname{templ}}$ contained in the training dataset.

\begin{definition}[Reproduction bias]
Given an output token sequence \(\mathbf{s}\), let \(e = \operatorname{seq}^{-1}(\mathbf{s})\) be the original expression that is represented by \(\mathbf{s}\). 
If \(\operatorname{strip}(e) \in E_{\operatorname{templ}}\), we say \(\mathbf{s}\) is a \emph{reproduction} of expressions seen during training.
\end{definition}

\begin{remark} \label{Rem:reproduction_bias}
We have defined and measured reproduction bias based on whether the training dataset contains an expression that is structurally equivalent to the generated one. However, one may argue that we should define and measure reproduction based on functional equivalence; there are many expressions that are structurally different but functionally equivalent (e.g., $x_1(x_1+x_2)$ and $x_1^2 + x_1 x_2$), and that such expressions should also be considered as equivalent expressions. To account for both structural and functional equivalence, we choose the operators as described below, so that any pair of structurally different expressions is also functionally different. For a more detailed discussion of the definition of reproduction bias, please refer to Appendix \ref{Sec:definition_reproduction_bias}.
\end{remark}

\paragraph{Method.}

We constructed a simplified training dataset consisting of $100$K equations. 
The allowed operator tokens were \texttt{add}, \texttt{sub}, \texttt{sin}, \texttt{cos}, \texttt{tan}, and \texttt{exp}, with up to $5$ independent variables per equation. We then trained a NeSymReS model on this dataset for $1,000$ epochs.
The variation of operators was limited due to two reasons. 
Firstly, the complicated training procedure of NeSymReS, where expressions with operators such as \texttt{mul} or \texttt{pow} are dynamically transformed and presented in different forms across epochs, makes it difficult to judge whether the model’s generated expressions are novel or memorized from training. Therefore, such operators were excluded.
Secondly, as explained in Remark \ref{Rem:reproduction_bias}, the choice of operators were restricted in order to account for both structural and functional equivalence when measuring reproduction bias.
The dataset size was also kept relatively small due to computational cost and to balance the size of the training data against the size of the search space.

For evaluation, we constructed two test datasets: \texttt{not\_included} and \texttt{baseline}, each containing $150$ expressions. 
For the \texttt{not\_included} set, we removed all $e_{\mathrm{templ}}$ appearing in the training data. In contrast, the \texttt{baseline} set was sampled directly from the generator $p_{\mathcal{E}}$ without any filtering.
We associated each expression with $100$ data points, generated in the same way as during training. We set the beam size to $5$ for this experiment.

To evaluate fitting performance, we used the $R^{2}$ score, defined as follows. Given a test equation, a set of $n$ data points $\left\{\mathbf x_{i}, y_{i}\right\}_{i=1}^{n}$, and the corresponding model predictions $\left\{\hat{y}_{i}\right\}_{i=1}^{n}$, the $R^{2}$ score is computed as:
\[
R^{2}=1-\frac{\sum_{i=1}^{n}\left(y_{i}-\hat{y}_{i}\right)^{2}}{\sum_{i=1}^{n}\left(y_{i}-\bar{y}\right)^{2}} \quad \text { where } \quad \bar{y}=\frac{1}{n} \sum_{i=1}^{n} y_{i}.
\]
Note that these $m$ evaluation points are distinct from the inputs provided to the model at the test time. By definition, $R^{2} \leq 1$, and values closer to $1$ indicate that the predicted outputs closely match the true equation. In our experiments, we counted the number of predictions whose $R^2$ exceeds the thresholds of $0.5, 0.9, 0.95, 0.99, 0.999, 0.9999$, and $0.9999$, respectively. This allows us to assess the model’s ability to fit the data under both moderate and stringent accuracy requirements.

\paragraph{Result.}
Figure \ref{Fig:baseline_vs_not_included} shows the results for NeSymReS under the simplified setting. The left figure demonstrates NeSymReS’s ability to generate novel expressions using the \texttt{not\_included} dataset. As this dataset comprises instances of $e_{\mathrm{templ}}$ unseen in the training data, the model is expected to produce previously unseen tree structures.
However, the result indicates that NeSymReS struggles to generate expression trees beyond the training data across varying epochs. After $1000$ epochs of training, over 97\% of the generated expression trees were direct copies from the training data, highlighting strong reproduction bias and revealing that the search space of NeSymReS is severely restricted. The right figure demonstrates how this reproduction bias negatively affects numerical accuracy, where NeSymReS’s fitting performance on the \texttt{not\_included} and \texttt{baseline} datasets is compared. The results indicate a substantial drop in performance for expressions whose tree structures are not present in the training data, compared with those sampled randomly. 
This suggests that for expressions not seen during training, the model's reproduction bias directly leads to poor numerical accuracy. This result also helps explain why NSR methods often fail to achieve high performance on expressions with many input variables; an increase in the number of input variables leads to an expanded search space, thereby increasing the likelihood that a given expression is absent from the training set.

\begin{figure}[!t]
  \begin{center}
  \centerline{\includegraphics[width=\columnwidth]{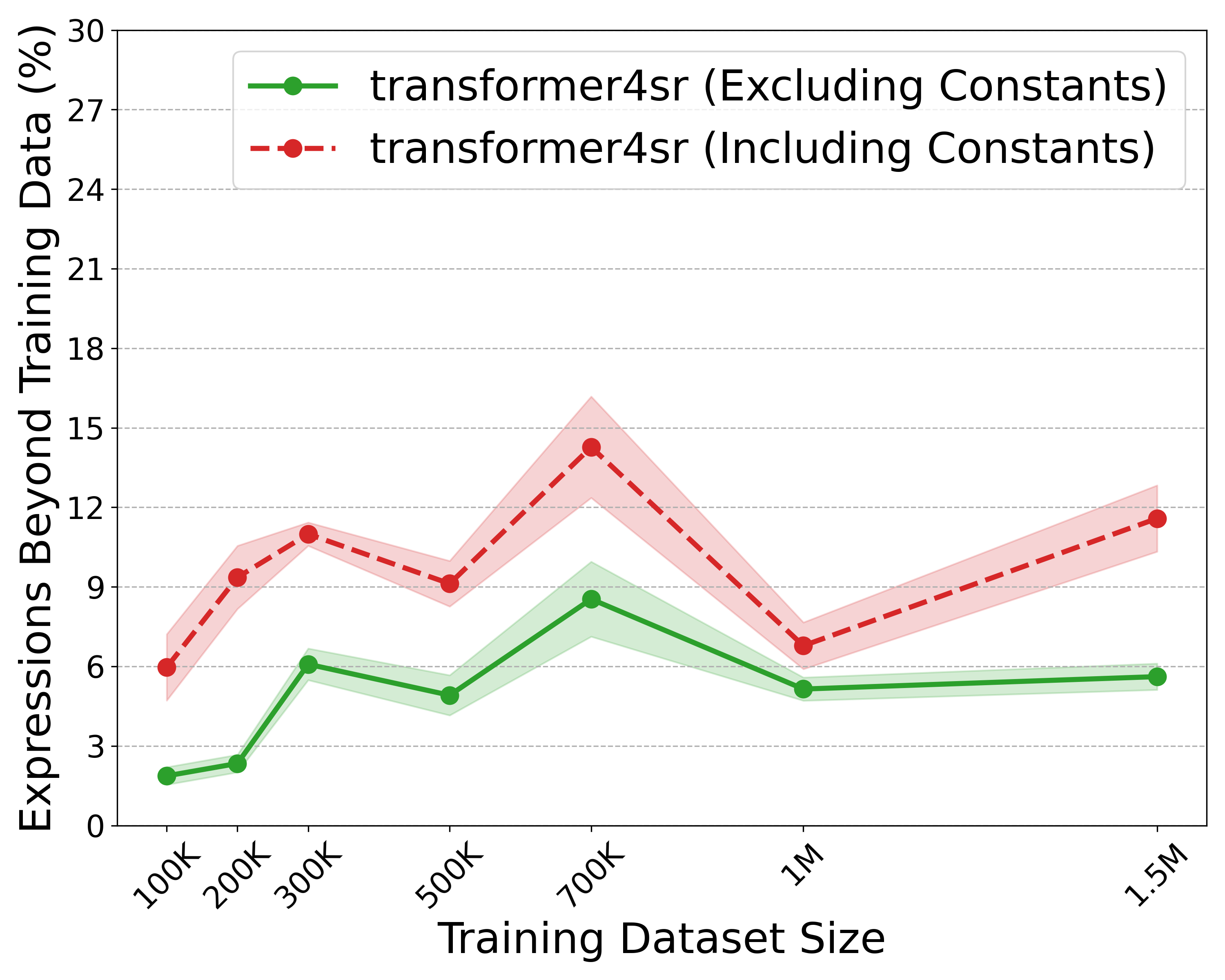}}
  \caption{Reproduction bias in transformer4sr with varying training dataset sizes. While small training dataset sizes (100K, 200K) exhibit
stronger reproduction bias, scaling the training dataset size does not necessarily mitigate reproduction bias after a certain limit.}
  \label{Fig:reproduction_bias_transformer4sr_size}
  \end{center}
\end{figure}

\subsection{Reproduction Bias in Practical Setting}\label{Ssec:experiment_practical}

Due to the complicated training procedure of NeSymReS, the above analysis was carried out in a simplified setting. To examine whether reproduction bias is a general phenomenon, we conducted an additional analysis in a more practical setting using transformer4sr \citep{lalande2023transformer}, a method similar to NeSymReS but with a simpler training process. In transformer4sr, no dynamic transformations of expressions are applied during training, which makes it much easier than in NeSymReS to verify whether the generated expressions are included in the training data. We were also able to analyze the novelty of the expressions by explicitly taking into account the positions of the constant placeholder tokens. Consequently, the definition of reproduction bias slightly differs from that described in Section \ref{Ssec:experiment_simple}, and we therefore provide additional clarification in Appendix \ref{Sec:definition_reproduction_bias}.

We followed the model architecture, number of epochs, operator selection, and inference strategies described in \citet{lalande2023transformer}. We constructed training datasets with varying size, with the largest one consisting of $1.5$M expressions based on the practical setting used in the original paper.
We used the full set of operator tokens, which are \texttt{add}, \texttt{mul}, \texttt{cos}, \texttt{log}, \texttt{exp}, \texttt{neg}, \texttt{inv}, \texttt{sqrt}, \texttt{sq} (squared), \texttt{cb} (cubed), and the number of independent variables was $6$. We constructed a test set similar to \texttt{not\_included} in the previous analysis, which consists of $300$ expressions that were not included in the training data.

Figure \ref{Fig:reproduction_bias_transformer4sr_size} shows the result for transformer4sr's ability to generate expressions beyond the training data. Even in a practical setting with $1.5$M training expressions, less than $12\%$ of the expressions generated by transformer4sr were novel expressions (taking into account the position of the constant placeholder tokens) beyond the training data, and less than $6\%$ of the expressions had novel tree structures (excluding constants). The result also demonstrates that the effect of scaling training data on mitigating reproduction bias tends to saturate, suggesting that reproduction bias is likely to persist even when the training data size is further increased.

\section{Can Test-time Strategies Mitigate Reproduction Bias?}\label{Sec:test_time_computation}

The results from the previous section indicate that the search space of NeSymReS is mostly confined to expressions seen during training due to reproduction bias. Since our theoretical analysis indicates that naively performing next-token prediction makes it difficult to generate novel expressions in a compositional manner, we investigated the possibility of devising inference-time computational techniques to reduce reproduction bias in this section.
Our hypothesis is that providing the model with hints about which tokens are appropriate could help steer the model to generate expressions that were not included in the training data. We begin by briefly introducing the three test-time strategies employed in our experiments. The detailed explanation for the strategies is presented in Appendix \ref{Sec:proposed_method_details}.

\subsection{Test-time Strategies}

\paragraph{Decoding with large beam size.} 
The beam search serves as the default decoding strategy employed by NeSymReS. During decoding, given a beam size of $b$, the decoding process generates $b$ candidate sequences via beam search. The constant placeholders of each candidate are subsequently optimized using the Broyden–Fletcher–Goldfarb–Shanno (BFGS) algorithm \citep{fletcher2000practical}. The expression exhibiting the highest numerical accuracy on the test data is then selected as the model’s output. While the experiments in Section \ref{Ssec:experiment_simple} used a beam size of $b=5$, in this section we conducted experiments with a larger beam size of $b=150$. Since increasing the beam size does not provide the model with any additional information, our hypothesis is that simply adopting a decoding strategy with a larger beam size will not alleviate reproduction bias.

\begin{figure}[!t]
  \begin{center}
  \centerline{\includegraphics[width=\columnwidth]{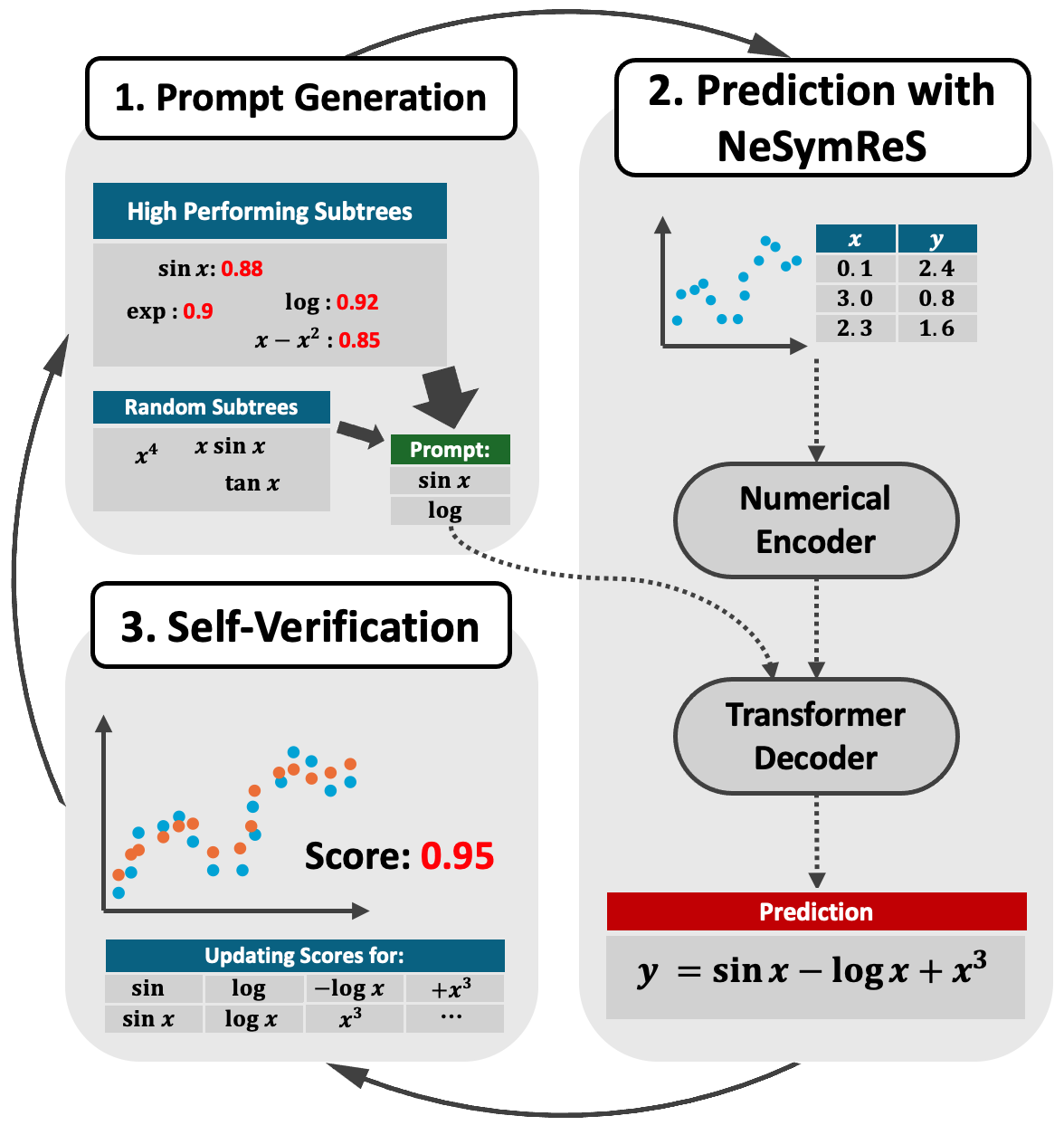}}
\caption{Overview of NSR-gvs's inference procedure. We first sampled subtrees from the candidate pool, then supplied them to the model together with numerical data.
Then, the generated prediction is numerically verified and the self-verification feedback is used to update the candidate pool. This procedure is performed repeatedly to generate better predictions over time.}
\label{Fig:proposed_method}
\end{center}
\end{figure}

\paragraph{Incorporating MCTS.}
TPSR \citep{shojaee2023transformer} is a method that leverages MCTS during decoding time.
In TPSR, the process starts by preparing a pre-trained NSR model (e.g., the NeSymReS model). Instead of relying on standard decoding methods like beam search, the method generates tokens using MCTS, where both the expansion and evaluation stages of MCTS leverage the pre-trained NSR model. In the expansion phase, to avoid unnecessary exploration, the set of expandable tokens is restricted to the top-$k_\mathrm{max}$ candidates based on the logits from the NSR model. During the evaluation phase, the NSR model first completes the remainder of the expression following the expanded token. The completed expression is then evaluated primarily based on its fitting accuracy, with additional consideration given to its complexity. In the experiments presented in this section, we used the default hyperparameter settings of TPSR as specified in the original paper; we set the number of rollouts to $r = 3$, the number of expandable tokens to $k_\mathrm{max} = 3$, and the beam size for expression completion to $b=1$.

\paragraph{NSR-gvs.}

TPSR provides feedback to the model by assigning a reward to each token, reflecting the quality or appropriateness of that token. In contrast, we hypothesized that incorporating feedback at the subtree level as well may have a positive effect on the model. To this end, we propose NSR-gvs, a method grounded in the following intuition: expressions that fit the same numerical data well are likely to share common substructures. 

We first trained a slightly modified version of the NeSymReS model, where the model takes subtrees as prompts and generates expressions that incorporate them. We achieved this by extracting subtrees from the ground-truth expressions and feeding them to the model together with numerical data during training. 
Figure \ref{Fig:proposed_method} illustrates the inference procedure of NSR-gvs. We generated multiple predictions iteratively by augmenting the model with varying prompts. For each iteration, we first sampled subtrees from a pool of candidate subtrees, which were extracted from high-performing expressions in previous predictions. To maintain output diversity, we also occasionally sampled subtrees from a random distribution. We then provided the sampled subtrees to the pre-trained model as prompts, along with the numerical data. After the model generates a prediction, it is automatically verified according to the fitting accuracy on the test data. Finally, the pool of candidate subtrees are updated based on the results of self-verification. This method can be formulated within the framework of reinforcement learning, and we provide a more detailed explanation in Appendix \ref{Sec:proposed_method_details}.

We conducted experiments in this section using $30$ iteration loops per expression, with the beam size $b = 5$ for generating each prediction. In addition, we experimented with a method that combines NSR-gvs with TPSR; in this approach, each prediction is produced via MCTS-based decoding instead of simple beam search.

\subsection{Results}

\begin{figure}[!t]
  \begin{center}
  \centerline{\includegraphics[width=\columnwidth]{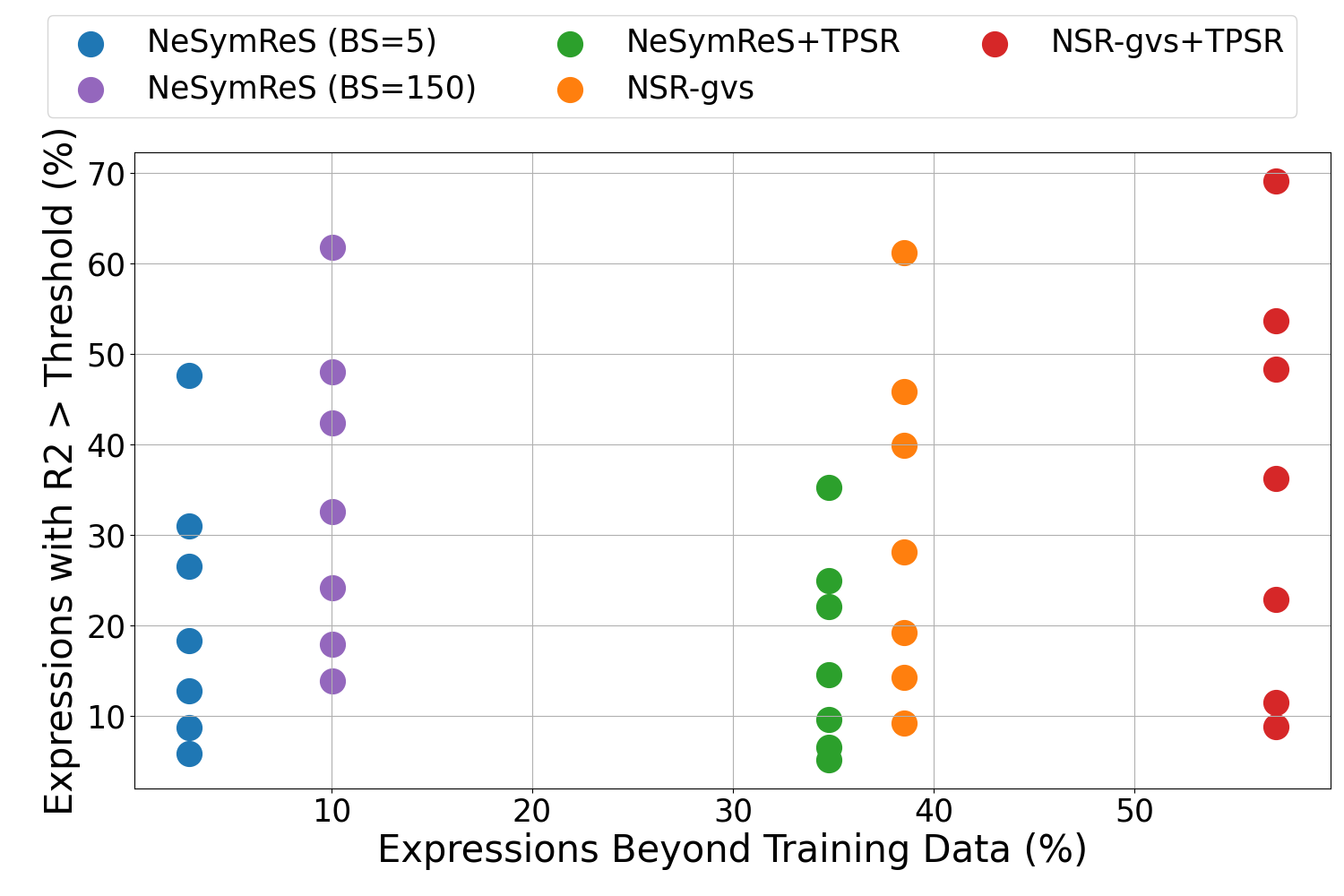}}
\caption{Evaluation of test-time strategies on the \texttt{not\_included} dataset. The x-axis represents the percentage of expressions generated that were not included in the training data. The y-axis shows the proportion of expressions that exceeded the $R^2$ thresholds of $0.5$, $0.9$, $0.95$, $0.99$, $0.999$, $0.9999$ and $0.99999$, respectively.}
  \label{Fig:controlled_scatter}
\end{center}
\end{figure}

We evaluated the impact of each test-time strategy on reproduction bias and numerical accuracy in an experimental setting. The experimental setup closely follows that described in Section \ref{Ssec:experiment_simple}. We trained a NeSymReS model and a prompt-augmented model for NSR-gvs with the same training dataset for the same number of epochs. We evaluated the strategies using the \texttt{not\_included} dataset, where we used the $R^2$ metric to evaluate numerical accuracy, and the number of novel expressions to evaluate reproduction bias.

Figure \ref{Fig:controlled_scatter} shows how the test-time strategies perform under the simplified setting. In terms of the ability to generate novel expressions, TPSR, NSR-gvs, and their combination demonstrate strong performance. These results imply that strategies involving the provision of additional information during inference (TPSR and NSR-gvs) are more effective in reducing reproduction bias. However, the result shows that high novelty in generated expressions does not necessarily imply high numerical accuracy. In some cases, acquiring the ability to generate novel expressions leads to a decrease in numerical accuracy (TPSR), whereas some strategies can improve numerical accuracy despite high reproduction bias (large beam size).
This suggests that, in some cases, Transformers struggle to leverage the additional information effectively. 
We therefore argue that providing additional information at test time in a way that is easy for the Transformer to leverage is important for developing a truly generalizable NSR approach. Viewed in this way, the use of subtrees at inference, as in the proposed method NSR-gvs, can be seen as a potentially valuable approach, since it contributes to mitigating reproduction bias and improving numerical accuracy.

\section{Conclusion}\label{Sec:conclusion}
In this work, we identified a major drawback of standard NSR models both theoretically and empirically. Our theoretical analysis shows that Transformers are incapable of generating expressions in a compositional way, while taking numerical data into account. We then examined the strategies that Transformers actually employ to generate expressions, and the results suggest that they mostly generate expressions copied from the training data, highly limiting the search space. Finally, we demonstrate that incorporating additional information to the model during test-time can reduce reproduction bias. In future work, we aim to build on the findings of this study to design symbolic regression methods that further improve generalizability.

\newpage

\bibliography{example_paper}
\bibliographystyle{icml2026}

\newpage
\appendix
\onecolumn
\appendix

\section{Details for NeSymReS} \label{Sec:detail_nesymres}
In this section, we present a detailed explanation for the study of NeSymReS that could not be fully explained in Section \ref{Sec:preliminary}. We discuss the details of the dataset generation process, the model architecture, and the training procedure.

\paragraph{Generating the dataset.} 
In the first step for generating the expression $e^\ast$, the unary-binary tree structure $T$ is generated randomly within the limits of a maximum depth of $6$. In the third step, the total number of constants added to the expression is also limited to a maximum of $6$. The binary and unary operators $\mathcal{O}_\mathrm{bin} \cup \mathcal{O}_\mathrm{un}$ are shown in Table \ref{Tab:operators}. Other hyparparameters are specified in Table \ref{Tab:hyperparams_nesymres}, where $\mathcal{LU}$ denotes the log-uniform distribution. After the expression is sampled, the symbolic manipulation library Sympy \citep{meurer2017sympy} is used to simplify any redundancy in the expression.

\begin{table}
  \caption{Operators used in NeSymReS}
  \label{Tab:operators}
  \centering
    \begin{tabular}{cc}
    \toprule
        Arity & Operators \\
        \midrule
        Unary & 
        \texttt{pow2, pow3, pow4, pow5}\\
        & \texttt{sqrt, log, exp}\\
        & \texttt{sin, cos, asin} \\
        \hline
        Binary & \texttt{add, sub, mul, div}\\
    \bottomrule
    \end{tabular}
\end{table}

\begin{table}
  \caption{Hyperparameters in NeSymReS's dataset generation}
  \label{Tab:hyperparams_nesymres}
  \centering
    \begin{tabular}{ccc}
    \toprule
        Name & Explanation & Value \\
        \midrule
        $d$ & Dimension for input variables & $5$ \\
        $n$ & Number of input points & Sampled from $\mathcal{U}(1,1000)$ \\
        $\mathcal{D}_\mathrm{mul}$ & Distribution over multiplicative constants & Sampled from $\mathcal{LU}(0.05, 10)$ \\
        $\mathcal{D}_\mathrm{add}$ & Distribution over additive constants & Sampled from $\mathcal{U}(-10,10)$ \\
        $\{x_{\mathrm{min},j}\}_{j=1}^d$ & Lower bound for sampling input variable & Sampled from $\mathcal{U}(-10,9)$ \\
        $\{x_{\mathrm{max},j}\}_{j=1}^d$ & Upper bound for sampling input variable & Sampled from $\mathcal{U}(x_{\mathrm{min},j}+1,10)$ \\
    \bottomrule
    \end{tabular}
\end{table}

\paragraph{Model architecture.}
The NeSymReS model consists of two architectural components: the numerical encoder $enc_\mathrm{num}$ and a decoder $dec$. The numerical encoder processes the numerical data $\mathcal{D}$, represented as a tensor of shape $(b,n,d)$, where $b$ denotes the batch size, $n$ the number of input points, and $d$ the sum of dependent and independent variables. First, an embedding layer converts the numerical data into a higher dimensional tensor $\mathcal{D}^{\prime}$ of shape $(b,n,h)$. This tensor is then processed by a $5$-layer set-transformer \citep{lee2019set} encoder that outputs a new tensor $Z_{num}$ of shape $(b,s,h)$, where $s$ denotes the number of embedding vectors produced by the encoder. The resulting tensor $Z_{num}$ is subsequently passed to the decoder $dec$, a five-layer standard Transformer decoder that auto-regressively generates the corresponding expression token by token. We set $b=200$, $h= 512$, and $s = 32$ for our experiments.

\paragraph{Details for training.}
During training, cross-entropy loss is used as the objective function, and teacher forcing \citep{sutskever2014sequence} is applied during next-token prediction. The AdamW \citep{loshchilov2017decoupled} optimizer is employed with an initial learning rate of $10^{-4}$. After $4000$ steps, the learning rate is adjusted proportionally to the inverse square root of the number of steps taken.

\section{Details for Test-time Strategies}\label{Sec:proposed_method_details}

This section is devoted to supplementing the details that were not fully covered in Section \ref{Sec:test_time_computation}. We first supplement our explanation of TPSR, followed by a detailed formulation of NSR-gvs.

\subsection{TPSR}

\begin{figure}[htbp]
\begin{center}
\centerline{\includegraphics[width=\columnwidth]{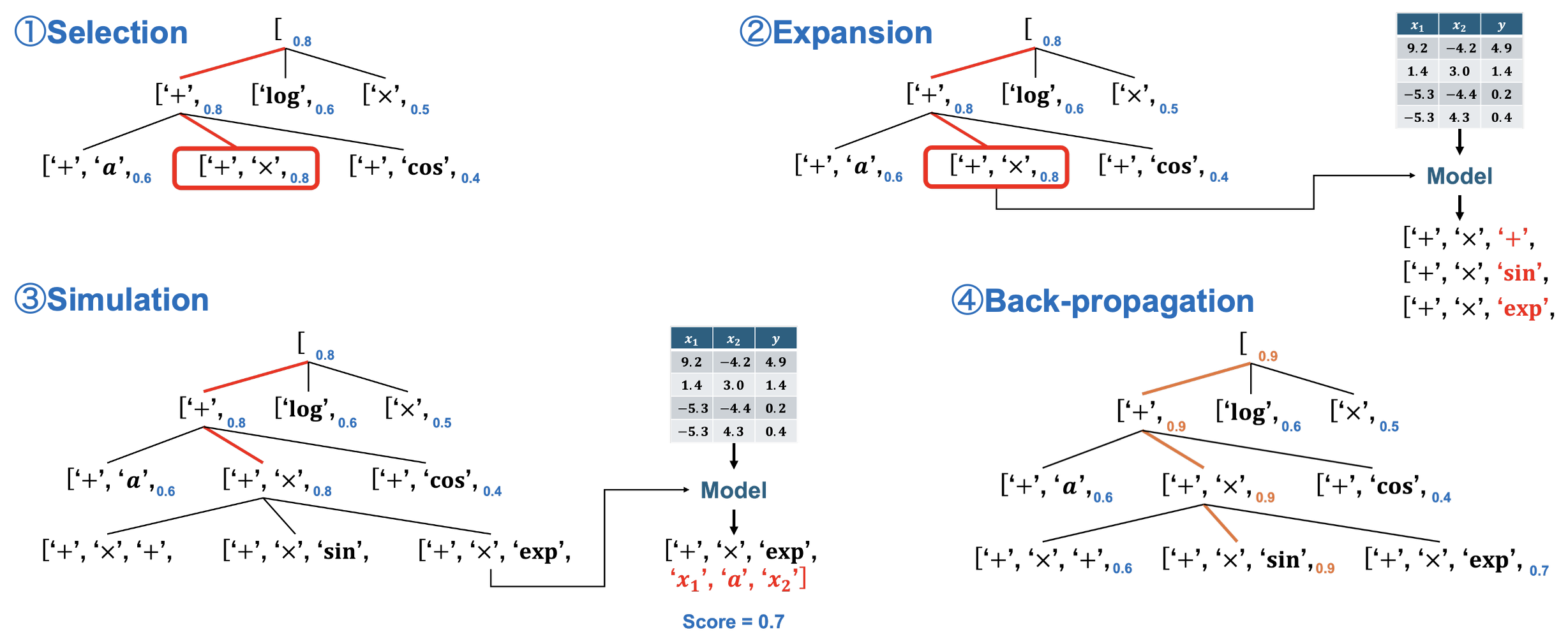}}
\caption{Overview of TPSR. Expressions are generated by MCTS with guidance from a NSR model. During the expansion phase, the next token to expand is decided based on the NSR model’s logits, and during the simulation phase, the NSR model simulates the hypotheses instead of random simulation.}
\label{Fig:tpsr_overview}
\end{center}
\end{figure}

TPSR is a method where expressions are generated by MCTS. MCTS consists of $4$ key steps, and here we supplement details for the steps that were not covered in the main text. We also provide an overview in Figure \ref{Fig:tpsr_overview}.

First, during the selection phase, the policy-guided upper confidence bound (P-UCB) heuristic \citep{silver2018general} is used, where the NSR model is used as the policy.

Then, as described in Section \ref{Sec:test_time_computation}, the NSR model decides the token to expand during the expansion phase and simulates the hypotheses during the simulation phase, playing a central role in the generation process.
After the simulation phase, the reward is mainly calculated based on the generated expression's numerical accuracy, with additional consideration given to its complexity. In TPSR, a hyperparameter $\lambda$ controls the balance between fitting accuracy and complexity. Given a set of $n$ data points $\left\{\mathbf x_{i}, y_{i}\right\}_{i=1}^{n}$, and a candidate prediction $\tilde{f}$, the reward $r(\tilde{f}(\cdot)\mid \mathbf x, \mathbf y)$ is calculated as follows:
\[
r(\tilde{f}(\cdot)\mid \mathbf x, \mathbf y) = \frac{1}{1 + \mathrm{NMSE}(\mathbf y, \tilde{f}(\mathbf x))} + \lambda \exp \left( -\frac{|\mathrm{seq}(\tilde f)|}{L} \right),
\]
where $\mathrm{seq}$ is the serialization mapping introduced in Section \ref{Sec:preliminary}, $L$ denotes the the model’s maximum sequence length, and $\mathrm{NMSE}$ represents the normalized mean square loss. In our work, we always set $\lambda$ to $0.01$, which is the default value in the original study of TPSR.

Finally, during the back-propagation phase, the scores for each node are updated by taking the maximum between the current score and the newly obtained reward value.

\subsection{NSR-gvs}
As described in Section \ref{Sec:test_time_computation}, NSR-gvs is a method that iteratively improves its predictions by providing expression subtrees as prompts to the model and receiving feedback through verification. In this section, we formulate the training and inference procedures of NSR-gvs within the framework of reinforcement learning.

\subsubsection{Training}

We first introduce a \emph{prompt-conditioned symbolic regressor} $S^\prime_{\boldsymbol\theta}$ defined by parameters $\boldsymbol\theta$, 
that maps a numerical dataset
$\mathcal{D}$ and an auxiliary prompt sequence $\mathbf p$ to a symbolic expression
$\hat e = S^\prime_{\boldsymbol\theta}(\mathcal{D},\mathbf p)$.
Learning aims to align $\hat e$ with the ground-truth expression $e^{\ast}$ underlying $\mathcal{D}$.
Among the elements of the synthetic training tuple $(e^{\ast},\mathcal{D},\mathbf p)$, the generation of $e^\ast$ and $\mathcal{D}$ is the same as explained in Section \ref{Sec:preliminary}. Here we specify how prompt sequences are constructed.

We first define $ \operatorname{extract}:\mathcal E\!\longrightarrow\!\mathcal{P}(\mathcal E)$ as a stochastic mapping, which assigns to each symbolic expression
$e \in \mathcal{E}$ a probability distribution over the subtrees of $e$. The space $\mathcal{P}(\mathcal E)$ denotes the power set of expressions.

Using this stochastic mapping, we first obtain $N$ subtrees $\{e^\prime_i \mid e^\prime_i \sim \operatorname{extract}(e^\ast), i = 1,2,\ldots,N\}$ from the ground-truth expression $e^\ast$. Then, each of the subtrees are converted to token sequences $\{\mathbf t_i \mid \mathbf t_i = \operatorname{seq}(e^\prime_i), i = 1,2,\ldots,N\}$using the serialization map $\operatorname{seq}$.
Given an expression $e^\ast$, we construct the prompt: 
\[
\mathbf{p}
  \;=\;
  (\tau_{\mathrm{start}}, \mathbf t_1, \tau_{\mathrm{end}}, \tau_{\mathrm{start}}, \mathbf t_2, \tau_{\mathrm{end}}, \ldots, 
  \tau_{\mathrm{start}}, \mathbf t_{N}, \tau_{\mathrm{end}}),
\]
where tokens $\tau_{\mathrm{start}}$ and $\tau_{\mathrm{end}}$ are partition tokens representing the beginning and end of each subtree representation.

Similar to the formulation in Section \ref{Sec:preliminary}, the predictive distribution
\(
  q^\prime_{\boldsymbol\theta}
  (\,\cdot\,\mid (\mathbf p,\mathbf s_{<j}),\mathcal D) 
\)
is realized by an encoder-decoder Transformer parametrized by $\boldsymbol\theta$. In NSR-gvs, however, the decoder is conditioned on $(\mathbf p,\mathbf s_{<j})$, which is the concatenation of the prompt $\mathbf p$ and previously emitted prefix $\mathbf s_{<j}$.

\subsubsection{Inference}
During inference, we guide the symbolic regressor $S^\prime_{\boldsymbol\theta}$ by prompting it with expression subtrees, which are obtained by a self-verification process. We formalize the inference-time mechanism of NSR-gvs within the framework of a Markov Decision Process (MDP). The core components of the MDP are defined as follows:

\paragraph{State space $\mathcal{S}$ and action space $\mathcal{A}$.}
The state at time $t$ is denoted by $s_t \in \mathcal{S}$.
The state is defined as $s_t = \{(e^\prime_i, z_i, c_i) \mid i = 1,2,\ldots,n_t\}$, which is a $n_t$-sized set comprising tuples of subtrees $e^\prime_i \in \mathcal{E}$, its corresponding \emph{verification scores} $z_i \in \mathbb{R}$, and its appearance count $c_i \in \mathbb{N}$. Therefore, the state space can be represented as $\mathcal{S} = \mathcal{P}(\mathcal{E} \times \mathbb{R} \times \mathbb{N})$.
The action $a_t \in \mathcal{A}$ is a prompt sequence described in the previous subsection. The action space is represented as $\mathcal{A}=(\Gamma \cup \{\tau_{\mathrm{start}},\tau_{\mathrm{end}}\})^\ast$.

\paragraph{Policy $\pi(a_t \mid s_t)$.}
We define a stochastic policy to sample an action $a_t$ from the current state $s_t$. An action is sampled following the procedure below.

First, we deterministically select a set of subtrees $E^\prime_{\mathrm{topk}}$, consisting of the top $k$ subtrees with the highest verification scores in state $s_t$, as follows:
\[
E^\prime_{\mathrm{topk}} \;=\; \{e^\prime_i \mid (e^\prime_i,z_i,c_i) \in s_{\mathrm{topk}},\,i= 1,2,\ldots,k\}, 
\quad \text { where } \quad s_{\mathrm{topk}} = \underset{s \subseteq s_t,\, |s| = k}{\operatorname{arg max}} \sum_{(e^\prime,z,c) \in s} z.
\]
Subsequently, we filter out subtrees whose corresponding score $z$ is smaller than a threshold value $z_\mathrm{thres}$. The purpose of this operation is to prioritize exploration over exploitation when the quality of obtained subtrees are poor.

Next, we construct a set \( E'_{\mathrm{rand}} \) by extracting \( k_{\mathrm{rand}} \) subtrees from expressions sampled from the expression generator \( p_{\mathcal{E}} \):
\[
E^\prime_{\mathrm{rand}}\;=\; \{e^\prime_i \mid e^\prime_i\sim \operatorname{extract}(e),\, e \sim p_{\mathcal{E}},\,i= 1,2,\ldots,k_\mathrm{rand}\}.
\]
Finally, we uniformly sample a set of subtrees from the merged set $E^\prime_{\mathrm{topk}} \cup E^\prime_{\mathrm{rand}}$ and convert them to tokens in the same way as during training time, resulting in a prompt sequence $a_t$. During sampling, we filter out subtrees whose token representation is longer than $l_\mathrm{max}$, and we sample subtrees until the total length of the subtrees' token representation exceeds the limit $l_t$. 

By sampling from both the self-verification-based set $E^\prime_{\mathrm{topk}}$ and the randomly obtained set $E^\prime_{\mathrm{rand}}$, the policy enables both exploration and exploitation. The hyperparameters $k$, $k_\mathrm{rand}$, $z_\mathrm{thres}$, $l_\mathrm{max}$, and $l_t$ characterize the policy.

\paragraph{Reward function $R(a_t, s_t)$ and transition probability $T(s_{t+1} \mid a_t, s_t)$.} 
After an action $a_t$ is sampled, it is provided to the prompt-conditioned symbolic regressor $S^\prime_{\boldsymbol\theta}$ together with numerical data $\mathcal{D}$. We compute the reward based on the numerical accuracy of the prediction $\hat e = S^\prime_{\boldsymbol\theta}(\mathcal{D},a_t)$.

Let $\mathcal{L}: \mathbb{R}^n \times \mathbb{R}^n \to \mathbb{R}$ be a metric to evaluate the difference between two vectors (in practice, we use the $R^2$ value described in Section \ref{Sec:reproduction_bias}). When $\mathcal{D} = \{(\mathbf x_i,\,y_i)\}_{i=1}^{n}$, the reward is computed as:
\[
R(a_t, s_t) = \mathcal{L}(\mathbf y, \hat e(\mathbf x)).
\]
Finally, we define the transition probability $T(s_{t+1} \mid a_t, s_t)$, determined by the following process. We denote by $\hat E^\prime$ the set comprising all subtree expressions of $\hat e$. For each subtree $\hat e^\prime$ in $\hat E^\prime$, we update $s_t$ so that the verification score of each subtree matches the average reward of all expressions that included the subtree, as described below.

\begin{enumerate}
\item If $\forall (e^\prime,z,c) \in s_t, \hat e^\prime \neq e^\prime$ holds, add the tuple $(\hat e^\prime, R(a_t, s_t), 1)$ to $s_t$.
\item If $\exists (e^\prime,z,c) \in s_t, \hat e^\prime = e^\prime$ holds, replace the tuple $(\hat e^\prime, z, c)$ with $(\hat e^\prime, \dfrac{cz + R(a_t, s_t)}{c + 1}, c + 1)$.
\end{enumerate}

The updated state serves as the state $s_{t+1}$ at the next timestep $t+1$.

\begin{table}
  \caption{Hyperparameters in NSR-gvs}
  \label{Tab:proposed_method_hyperparams}
  \centering
    \begin{tabular}{ccc}
    \toprule
        Name & Explanation & Value \\
        \midrule
        $k$ & Size of high-scored subtree set $E^\prime_{\mathrm{topk}}$& $39$ \\
        $k_\mathrm{rand}$ & Size of randomly sampled subtree set $E^\prime_{\mathrm{rand}}$ & $9$ \\
        $z_\mathrm{thres}$ & Threshold value for high-scored subtrees & $0.213$ \\ 
        $l_\mathrm{max} $ & Maximum length of a subtree's representation & $9$ \\
        $l_t$ & Total length of the subtrees' representation & Sampled from $\mathcal{U}(0,\lfloor15.58 + 0.42t \rfloor)$ \\
    \bottomrule
    \end{tabular}
\end{table}

The overall algorithm during inference-time is detailed in \ref{Alg:inference}. For the hyperparameters that characterize the policy, we use the values shown in Table \ref{Tab:proposed_method_hyperparams}, which were tuned via Bayesian optimization on $5$ randomly generated expressions. The function $\lfloor \cdot \rfloor$ indicates the floor function, which rounds down the input to its nearest integer.

\begin{algorithm}
\caption{Inference-time Algorithm}\label{Alg:inference}
\begin{algorithmic}

\FUNCTION{Verify($e$, $\mathcal{D}$)}
    \STATE $(X, \mathbf{y}) \gets \mathcal{D}$
    \STATE $\mathbf{\hat{y}} \gets e(X)$
    \STATE Compute $R^2$ score between $\mathbf{y}$ and $\mathbf{\hat{y}}$
    \STATE \textbf{return} $R^2$
\ENDFUNCTION
\STATE
\FUNCTION{Update($e$, $s_t$, $R^2$)}
    \STATE $E^\prime_p \gets$ Partial expressions extracted from $e$
    \STATE $s_{t+1} \gets [\,]$ 
    \FOR{$(e_p, z, c)$ in $s_t$}
        \IF{$e_p \in E^\prime_p$}
            \STATE $z \gets \dfrac{cz + R^2}{c + 1}$
            \STATE $c \gets c + 1$
        \ENDIF
        \STATE Append $(e_p, z, c)$ to $s_{t+1}$
    \ENDFOR
    \STATE \textbf{return} $s_{t+1}$
\ENDFUNCTION
\STATE
\MAIN{NSR-gvs-Inference($\mathcal{D}$)}
    \STATE $s_1 \gets [\,]$ 
    \STATE $e_{\mathrm{best}} \gets \operatorname{None}$
    \STATE $R^2_{\mathrm{best}} \gets -\infty$
    \FOR{$t \gets 1$ to $T$}
        \IF{$t = 1$}
            \STATE $a_t \gets [\,]$
        \ELSE
            \STATE $E^\prime_\mathrm{topk} \gets$ Top $k$ expressions in $s_t$ with high score
            \STATE Filter out expressions in $E^\prime_\mathrm{topk}$ whose corresponding score $z < z_\mathrm{thres}$
            \STATE $E^\prime_\mathrm{rand} \gets$ Randomly sampled partial expressions
            \STATE $E^\prime_{\mathrm{merged}} \gets E^\prime_{\mathrm{topk}} \cup E^\prime_{\mathrm{rand}}$
            \STATE $a_t \gets$ Uniformly sampled subset from $E^\prime_{\mathrm{merged}}$, converted to tokens
        \ENDIF
        \STATE $\mathbf{s} \gets \operatorname{Transformer}(\mathcal{D}, a_t)$
        \STATE Convert sequence $\mathbf{s}$ to expression $e$
        \STATE $R^2 \gets \operatorname{Verify}(e, \mathcal{D})$
        \STATE $s_{t+1} \gets \operatorname{Update}(e, s_t, R^2)$
        \IF{$R^2 > R^2_{\mathrm{best}}$}
            \STATE $e_{\mathrm{best}} \gets e$
            \STATE $R^2_{\mathrm{best}} \gets R^2$
        \ENDIF
    \ENDFOR
    \STATE \textbf{return} $e_{\mathrm{best}}$
\ENDMAIN
\end{algorithmic}
\end{algorithm}

\newpage

\section{Details for Experiments, Implementation, and use of LLMs}\label{Sec:experiment_details}

In this section, we describe the details for the experiments conducted in Section \ref{Sec:reproduction_bias}, \ref{Sec:test_time_computation}, and \ref{Sec:additional_experiments}. We also provide details regarding our implementation and the computational resources used in our experiments.

We provide the model with $100$ data points in all experiments. We selected the range of the data support as follows: for the AI Feynman dataset, we used the support defined by the dataset itself. For all other datasets, we sampled the support range using the same procedure as used when generating the training data.
For error bars, we report the standard deviation across three different random seeds. For the method combining NSR-gvs and TPSR, however, we conducted experiments with only a single seed due to the long inference time.
Our implementation for data generation, model training, and related components is based on the original implementation of NSRwH \footnote{\url{https://github.com/SymposiumOrganization/ControllableNeuralSymbolicRegression}}. For the transformer4sr and TPSR experiments, we used the official implementation provided by the authors \footnote{\url{https://github.com/omron-sinicx/transformer4sr}}, \footnote{\url{https://github.com/deep-symbolic-mathematics/TPSR}}. For both implementations, we used the version of the implementation that was available on May 15. 2025. We trained and tested the model on a single NVIDIA A100 GPU. Training requires approximately 24 hours either for $1000$ epochs on a dataset with $100,000$ expressions or for $10$ epochs on a dataset with $10$ million expressions. The time required to generate a single expression at test time is less than one minute when using only NeSymReS or NSRwH, approximately $3$ to $10$ minutes with TPSR or NSR-gvs, and around $2$ to $5$ hours when combining TPSR with NSR-gvs.

We used large language models (LLMs) to aid writing and coding, where we mainly used Gemini 2.5 Flash and GPT-5 to generate code and check on errors in writing.

\section{Discussion Concerning the Definition of Reproduction Bias}
\label{Sec:definition_reproduction_bias}
We first restate the definition of reproduction bias introduced in Section \ref{Ssec:experiment_simple}. This definition was introduced to discuss the novelty of expressions in the NeSymReS setting without considering constants. 

\begin{definition}[Reproduction bias] \label{Def:reproduction_bias_nesymres_appendix}
Given an output token sequence \(\mathbf{s}\), let \(e = \operatorname{seq}^{-1}(\mathbf{s})\) be the original expression that is represented by \(\mathbf{s}\). 
If \(\operatorname{strip}(e) \in E_{\operatorname{templ}}\), we say \(\mathbf{s}\) is a \emph{reproduction} of expressions seen during training.
\end{definition}

Here, we define reproduction bias in the transformer4sr setting. NeSymReS and transformer4sr are different in the way expressions are stored in the training data; in NeSymReS, constant-free expressions \(e_{\operatorname{templ}}\) are collected to form a set of training expressions \(E_{\operatorname{templ}}\), whereas in transformer4sr, full expressions including constants, denoted by \(e^*\), are stored, and the training dataset consists of the set \(E^*\) of such expressions. 

For transformer4sr, we measured expression novelty both with and without considering constants. If we consider the set 
\[
E^*_{\operatorname{templ}} = \{e^*_{\operatorname{templ}}\;|\;e^*_{\operatorname{templ}} = \operatorname{strip(e^*)}\} \quad \text {for all} \quad e^* \in E^*
\]
the definition of reproduction bias without considering constants would be the same as Definition \ref{Def:reproduction_bias_nesymres_appendix}. When considering constants, the definition of reproduction bias would be slightly changed as follows:

\begin{definition}[Reproduction bias considering constants] \label{Def:reproduction_bias_transformer4sr_appendix}
Given an output token sequence \(\mathbf{s}\), let \(e = \operatorname{seq}^{-1}(\mathbf{s})\) be the original expression that is represented by \(\mathbf{s}\). 
If \(e \in E^*\), we say \(\mathbf{s}\) is a \emph{reproduction} of expressions seen during training.
\end{definition}

Next, we discuss the problem of structural equivalence and functional equivalence in the transformer4sr setting. While we have defined reproduction bias based on structural equivalence, one may argue that we should define and measure reproduction based on functional equivalence. For example, $x_1(x_1+x_2)$ and $x_1^2 + x_1 x_2$ are different expressions from a structural perspective, but they are functionally equivalent. We stated in Remark \ref{Rem:reproduction_bias} that for the NeSymReS setting, structural equivalence and functional equivalence can be treated the same due to the operator choices (there is no possibility that, for example, $x_1(x_1+x_2)$ is generated when $x_1^2 + x_1 x_2$ is included in the training dataset). In the transformer4sr setting, this is not always the case, since we do not restrict the operators in this setting. 

However, even if we were to redefine reproduction bias from the perspective of functional reproduction bias, the claims of Section \ref{Ssec:experiment_practical} would remain unaffected. This is because every expression regarded as novel under the definition of functional reproduction bias is already also regarded as novel under the definition of structural reproduction bias. Since the claim of \ref{Ssec:experiment_practical} is that the proportion of expressions classified as novel is small for naive NSR methods, adopting the definition of functional reproduction bias would only further reduce that proportion, without altering the direction of the conclusions.

\section{Additional Related Work} \label{Sec:additional_related_work}

In this section, we describe symbolic regression methods other than NSR. Specifically, we provide explanation for methods that use GP, brute-force algorithms, and reinforcement learning.

The GP framework is a traditional and widely used framework for solving symbolic regression. The GP algorithm is a method based on evolutionary computation; initially, several mathematical expressions are formed randomly, and subsequently the expressions are ``evolved'' by operations such as recombining two expressions, mutating an expression, and eliminating inappropriate expressions \citep{burlacu2020operon, schmidt2009distilling, virgolin2019linear,cranmer2023interpretable}. 

An example of using brute-force algorithms for symbolic regression is AI Feynman \citep{udrescu2020ai, udrescu2020ai2}. In AI Feynman, neural networks were used to identify properties such as symmetry and separability within given numerical data. These properties were then used to recursively simplify the problem, ultimately reducing it to a form amenable to brute-force solutions.

\citet{petersen2019deep} proposed Deep Symbolic Regression (DSR), a method that uses reinforcement learning to tackle symbolic regression. In this approach, the authors used a recurrent neural network (RNN) to generate equations as token sequences, with the parameters that govern the selection of the token learned through reinforcement learning. Studies such as Symbolic Physics Learner \citep{sun2022symbolic} and Reinforcement Symbolic Regression Machine \citep{xu2024reinforcement} also use reinforcement learning, where Monte Carlo Tree Search (MCTS) is applied to discover expressions.

Some studies combine several approaches for symbolic regression. For example, neural-guided genetic programming \citep{mundhenk2021symbolic} integrates DSR and genetic programming (GP), while the Unified DSR Framework \citep{landajuela2022unified} combines GP, AI Feynman, DSR, linear models, and NSR.

\section{Additional Experiments} \label{Sec:additional_experiments}

\subsection{Do Novel Expressions Contribute to Improvements in Numerical Accuracy?}
In Section \ref{Sec:test_time_computation}, we saw that providing additional information to the model during inference can lead to generation of novel expressions. However, we also demonstrated that mitigating reproduction bias does not necessarily lead to better numerical accuracy. In this section, we analyzed how much the novel expressions generated under each test-time strategy (including NSRwH) contribute to improvements in numerical accuracy, and present the corresponding results in Table \ref{Tab:novelty_vs_accuracy}. ``Novel'' indicates that the generated expression does not appear in the training data, while ``Not Novel'' means it does. The values indicate the percentage of expressions that satisfy each condition.

\begin{table}
  \caption{Breakdown of generated expressions by novelty and high accuracy $(R^2 > 0.99)$ across test-time strategies}
  \label{Tab:novelty_vs_accuracy}
  \centering
    \begin{tabular}{ccccc}
    \toprule
        Test-time Strategy & Novel, & Novel, & Not Novel, & Not Novel, \\
        ~ & $R^2 > 0.99$ & $R^2 \leq 0.99$ & $R^2 > 0.99$ & $R^2 \leq 0.99$ \\
        \midrule
        NeSymReS (BS=1) & 0.45 & 2.89 & 12.92 & 83.74 \\
        NeSymReS (BS=5) & 0.67 & 2.23 & 17.59 & 79.52 \\
        NeSymReS (BS=50) & 4.23 & 2.23 & 25.17 & 68.38 \\
        NeSymReS (BS=100) & 4.45 & 2.45 & 26.72 & 66.37 \\
        NeSymReS (BS=150) & 6.25 & 3.79 & 26.33 & 63.62 \\
        \midrule
        \textbf{NeSymReS + TPSR (BS=1)} & 0.45 & \textbf{27.17} & 14.70 & 57.69 \\
        \textbf{NeSymReS + TPSR (BS=3)} & 0.67 & \textbf{29.62} & 18.04 & 51.67 \\
        \textbf{NeSymReS + TPSR (BS=5)} & 2.02 & \textbf{34.31} & 18.16 & 45.51 \\
        \midrule
        \textbf{NSR-gvs (BS=1)} & 4.91 & \textbf{38.18} & 14.05 & 42.86 \\
        \textbf{NSR-gvs (BS=3)} & 6.67 & \textbf{36.00} & 15.56 & 41.78 \\
        \textbf{NSR-gvs (BS=5)} & 8.68 & \textbf{30.51} & 19.38 & 41.43 \\
        \midrule
        \textbf{NSR-gvs + TPSR (BS=1)} & 12.75 & \textbf{44.30} & 23.46 & 19.46 \\
        \midrule
        NSRwH (Complexity, BS=5) & 2.23 & 5.80 & 15.18 & 76.78 \\
        NSRwH (Symmetry, BS=5) & 1.11 & 2.90 & 14.93 & 81.06 \\
        \textbf{NSRwH (Positives, BS=5)} & 2.46 & \textbf{24.83} & 6.71 & 66.00 \\
        NSRwH (Negatives, BS=5) & 0.89 & 4.68 & 13.81 & 80.62 \\
        \textbf{NSRwH (All, BS=5)} & 4.90 & \textbf{47.66} & 4.23 & 43.21 \\
    \bottomrule
    \end{tabular}
\end{table}

The results show that for test-time strategies that are capable of mitigating reproduction bias (strategies shown in bold), a large proportion of generated novel expressions do not perform well in terms of numerical accuracy. Especially for TPSR, hardly any of the novel expressions exhibit high numerical accuracy. This indicates the difficulty of generating appropriate expressions from an expanded search space. However, for strategies using NSR-gvs, novel expressions contribute to high accuracy to some extent, showing that additional information can be beneficial for both mitigating reproduction bias and improving numerical accuracy in some occasions.

\subsection{Trade-off Between Performance and Computational Cost}
The results in Section \ref{Sec:test_time_computation} show how the relationship between reproduction bias and numerical accuracy differ between various test-time strategies. However, test-time strategies also differ in terms of the computational cost required to generate an expression. In this section, we aim to better understand each test-time strategy by analyzing the trade-off between performance and the computational cost of expression generation. We also varied the beam size during decoding for NeSymReS, TPSR, and NSR-gvs for a more comprehensive analysis. We tested under the controlled setting described in Section \ref{Sec:reproduction_bias}, using the \texttt{not\_included} dataset as the test dataset.

To measure the computational cost, we followed the approach of \citet{shojaee2023transformer} and used the number of candidate expressions generated by the model during the generation of a single equation. For example, this value corresponds to the beam size in NeSymReS, the number of total rollouts multiplied by beam size in TPSR, and the number of iteration loops multiplied by beam size in NSR-gvs.

\begin{figure}[htbp]
\begin{center}
\centerline{\includegraphics[width=\columnwidth]{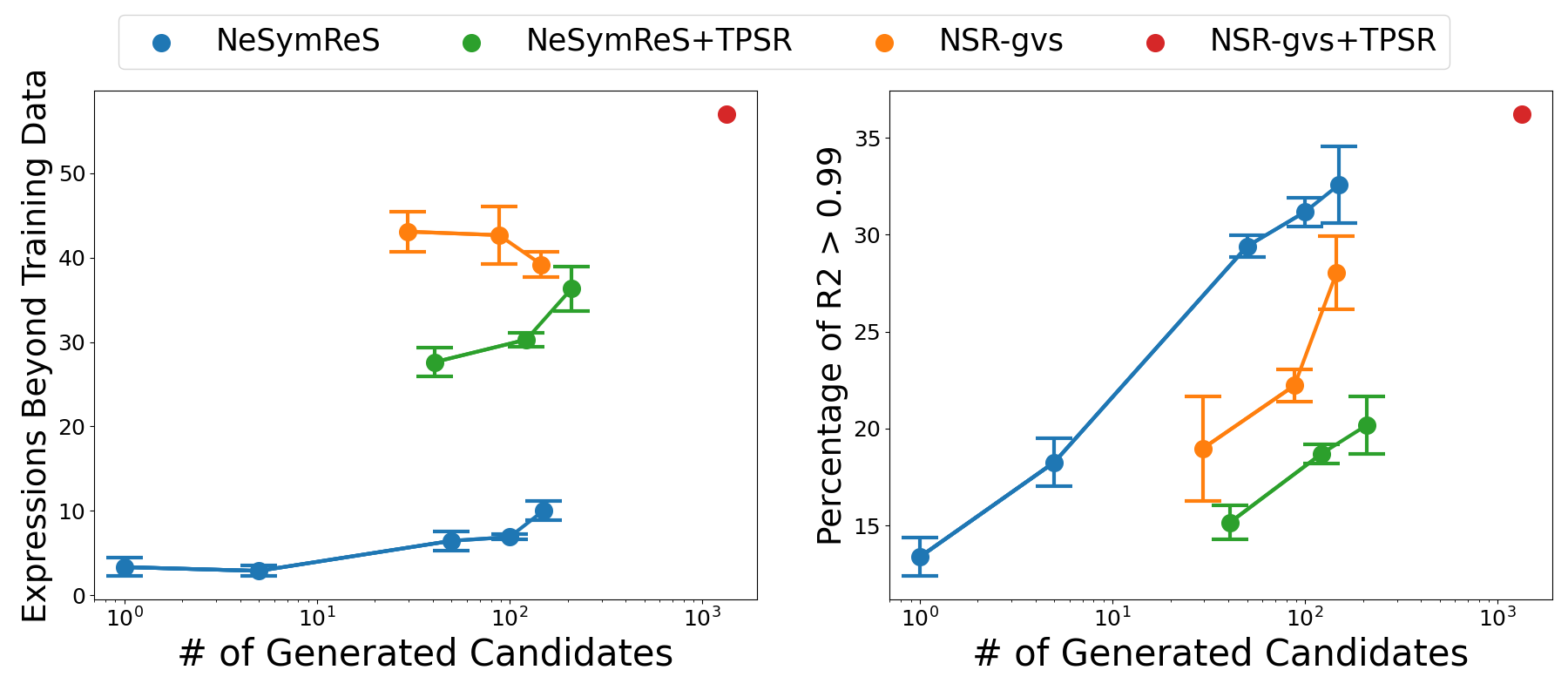}}
\caption{Trade-off between performance and computational cost for different test-time strategies. We varied the beam sizes for each model as follows: $\{1, 5, 50, 100, 150\}$ for NeSymReS, and $\{1, 3, 5\}$ for both NeSymReS+TPSR and NSR-gvs. For NSR-gvs+TPSR, we only experimented with beam size set to $1$. The left figure shows the trade-off between the ability to generate expressions and computational cost, while the right figure shows the trade-off between numerical accuracy and computational cost.}
\label{Fig:tradeoff}
\end{center}
\end{figure}

We present the results in Figure \ref{Fig:tradeoff}. It can be observed that, unlike NeSymReS—where larger beam size yields only limited reduction in reproduction bias—TPSR and NSR-gvs achieve notable reductions in reproduction bias at comparable computational costs. However, in terms of numerical accuracy, simply increasing the beam size in NeSymReS yields better performance than using NSR-gvs or TPSR at a comparable computational cost. The results support the conclusion in Section \ref{Sec:test_time_computation} that the reduction of reproduction bias is only weakly correlated with numerical accuracy.

\subsection{Can NSRwH Also Mitigate Reproduction Bias?}

\begin{figure}[htbp]
\begin{center}
\centerline{\includegraphics[width=0.5\columnwidth]{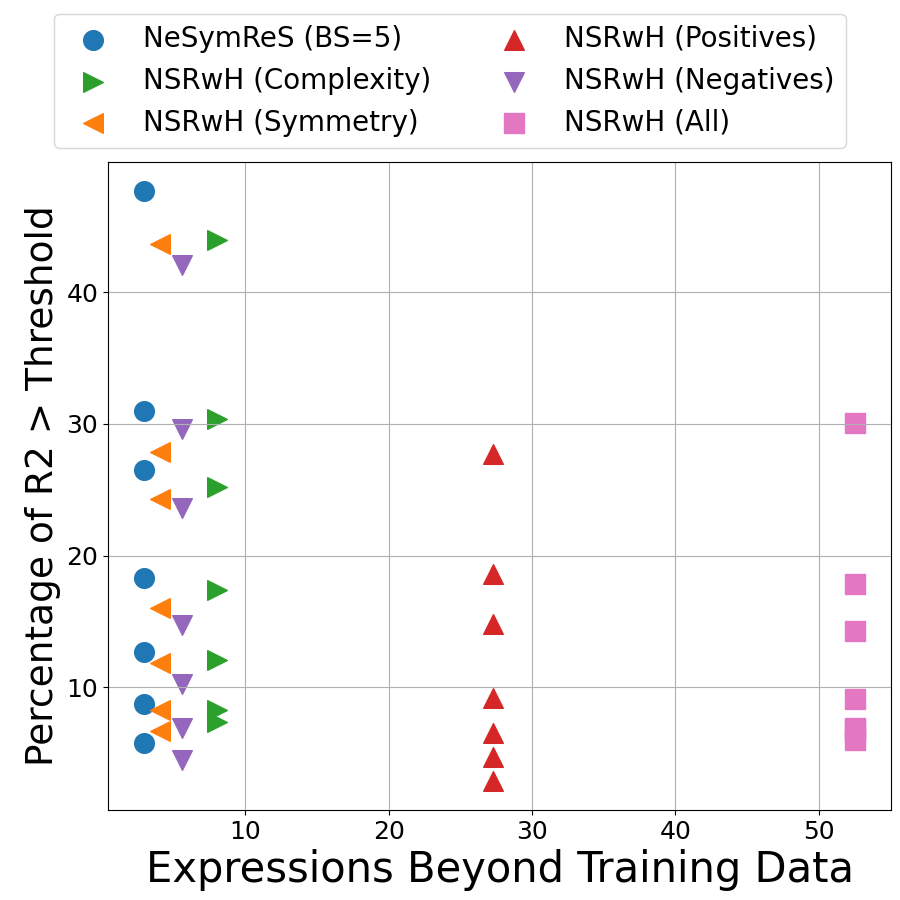}}
\caption{Evaluation of NSRwH on the \texttt{not\_included} dataset. The x-axis represents the percentage of expressions generated that were not included in the training data. The y-axis shows the proportion of expressions that exceeded the $R^2$ thresholds of $0.5$, $0.9$, $0.95$, $0.99$, $0.999$, $0.9999$ and $0.99999$, respectively.}
  \label{Fig:nsrwh_scatter}
\end{center}
\end{figure}

When researchers in fields of natural sciences or engineering model their experimental data, they often make use of prior knowledge. For example, scientists may anticipate a symmetry between variables or predict that a particular operator appears in the mathematical laws describing the data. NSRwH \citep{bendinelli2023controllable} is a method that enables incorporating such prior knowledge into the NeSymReS model. The types of prior knowledge provided to the model include the following:
\begin{itemize}
\item \textbf{Complexity.} The complexity of an expression is defined by the number of tokens used in the expression's token sequence. The model is provided with the complexity of the ground-truth expression.
\item \textbf{Symmetry.} The presence or absence of symmetry among the input variables is provided to the model.
\item \textbf{Positives.} Subtrees appearing in the ground-truth expression are provided to the model. Additionally, the value of constants appearing in the ground-truth expression may also be provided.
\item \textbf{Negatives.} Subtrees that do not appear in the ground-truth expression are provided to the model.
\end{itemize}
In NSRwH, prior knowledge is encoded by an additional symbolical encoder $enc_\mathrm{sym}$. The output of the symbolical encoder is summed together with the output of NeSymReS's numerical encoder and is fed to the decoder.

While prior knowledge is required beforehand to use NSRwH, it is a method that provides the model with additional information during inference, similar to TPSR and NSR-gvs. In this section, we test whether NSRwH can mitigate reproduction bias when prior knowledge is provided. We obtained a NSRwH model by finetuning the NeSymReS model that we trained in Section \ref{Sec:test_time_computation}. We froze the numerical encoder of the NeSymReS model, attached a symbolical encoder, and finetuned the model for $250$ epochs. We used the same training dataset as in Section \ref{Sec:reproduction_bias} consisting of $100,000$ expressions; however, during fine-tuning, prior knowledge was extracted from the ground-truth expressions and fed into the symbolic encoder. At test time, we evaluated the NSRwH model under settings where each type of prior knowledge is provided individually, as well as under a setting where all types of prior knowledge are provided simultaneously. We follow the default settings of NSRwH to determine the amount of prior knowledge provided during test-time, and we used the \texttt{not\_included} dataset as the test dataset. We set the beam size to $5$ and compare the results with those of NeSymReS, which is also configured with a beam size of $5$.

Figure \ref{Fig:nsrwh_scatter} shows the results for this experiment. While providing complexity, symmetry, or absent subtrees mitigates reproduction bias only to a limited extent, providing appearing subtrees or providing all properties significantly mitigates reproduction bias. However, we also observe that the numerical accuracy of NSRwH decreases when provided with appearing subtrees or with all properties. This indicates a limitation of NSRwH when dealing with data not included in the training set. The results also show that not all kinds of additional data are effective for mitigating reproduction bias.

\subsection{Numerical Accuracy in Practical Settings}

\begin{figure}[htbp]
\begin{center}
\centerline{\includegraphics[width=\columnwidth]{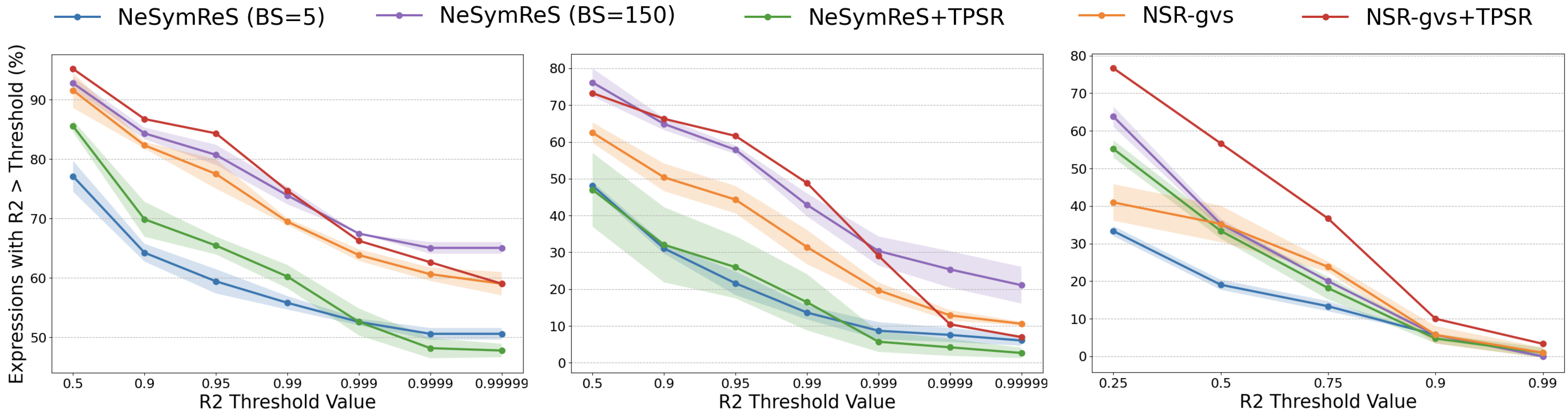}}
\caption{Comparison of test-time strategies on the \texttt{AI-Feynman} dataset. The general pattern remains consistent to the simplified setting, where large beam search and NSR-gvs+TPSR perform well, and NeSymReS+TPSR perform poorly.}
\label{Fig:performance_practical}
\end{center}
\end{figure}

Since the experiments in the main text were conducted under a relatively small training dataset with restricted operators, we additionally evaluated and compared the numerical performance of the test-time strategies under conditions that better reflect practical applications. A total of $10$ million expressions were used to construct the training dataset, employing all operators described in Section \ref{Sec:preliminary} without any restriction on operator types. We trained both a NeSymReS model and a prompt-augmented model on this dataset for $10$ epochs. For the test datasets, we prepared the following three sets:

\begin{itemize}
    \item \texttt{AI-Feynman.} This dataset consists of $91$ equations with up to five independent variables, extracted from the AIFeynman database \citep{udrescu2020ai}. It is commonly used in various studies to assess the performance of symbolic regression methods.
    \item \texttt{only\_five\_variables\_nc.} This dataset consists of expressions containing exactly five independent variables, making it a challenging dataset. The ``nc'' designation indicates that the expressions do not include constants, which simplifies the problems slightly; however, it remains more difficult than the first dataset. The dataset was constructed by sampling expressions from $p_\mathcal{E}$, filtering for expressions that include exactly five variables, and finally deleting its constants. This dataset is derived from the study of NSRwH \citep{bendinelli2023controllable}, and we use the first $100$ expressions for evaluation.
    \item \texttt{black-box}. We also evaluated on numerical data collected from the real world, whose ground-truth expressions do not exist. We extracted 35 expressions from the black-box dataset in SR-Bench \citep{la2021contemporary} whose number of independent variables are five or less. The data are often noisy and may be sampled from a range different from the numerical data that the models were trained on, making the task challenging for the test-time computation methods.
\end{itemize}

Figures \ref{Fig:performance_practical} demonstrates how the different test-time strategies perform under \texttt{AI-Feynman}, \texttt{only\_five\_variables\_nc}, and \texttt{black-box} datasets. TPSR relatively performs slightly better than in the controlled setting; however, the general pattern of numerical accuracy remains consistent. These results demonstrate that, NSR-gvs is able to improve performance in practical settings, including those with noisy data.
However, it can also be seen that even in practical settings, test-time strategies that mitigate reproduction bias do not always result in better performance.

\subsection{Varying the Number of Epochs in transformer4sr}

\begin{figure}[htbp]
\begin{center}
\centerline{\includegraphics[width=0.5\columnwidth]{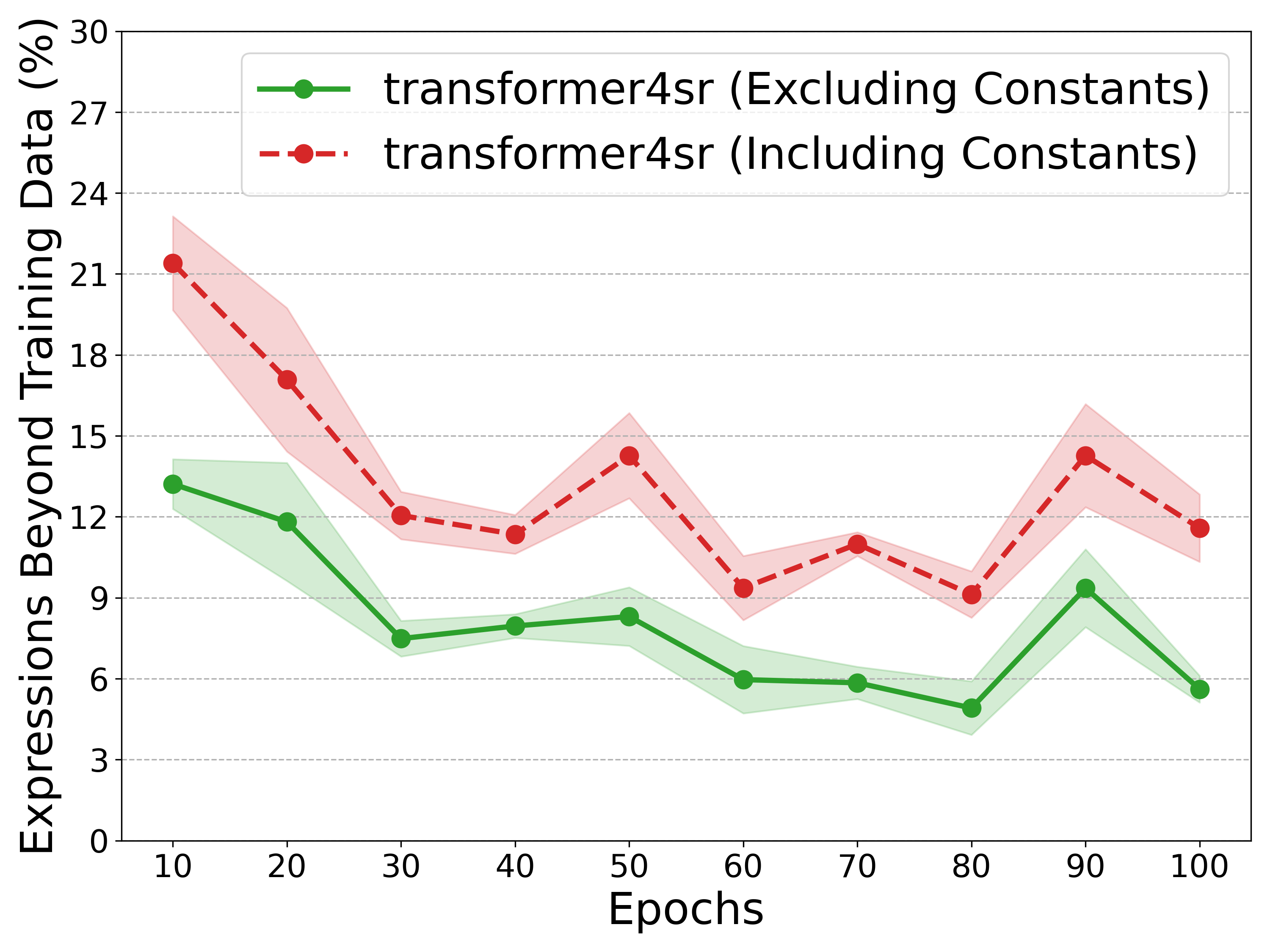}}
\caption{Reproduction bias in transformer4sr with varying epochs. Throughout the training process, the majority of both expression trees (both including and excluding constants) are copied directly from the training data.}
\label{Fig:transformer4sr_reproduction_bias}
\end{center}
\end{figure}

In Section \ref{Sec:reproduction_bias}, we tested whether reproduction bias occurs in the practical setting of transformer4sr, while varying the training dataset size.  
Here, we present the results of an analysis of reproduction bias as a function of the number of training epochs. In the experiments reported in Section \ref{Sec:reproduction_bias}, the number of training epochs was fixed to $100$ following the original paper; therefore, we evaluate reproduction bias by varying the number of epochs from $10$ to $100$ in increments of $10$. The results are shown in Figure \ref{Fig:transformer4sr_reproduction_bias}.  While reproduction bias is slightly lower at smaller numbers of epochs, it remains pronounced throughout the training process, indicating that reproduction bias cannot be attributed merely to underfitting or overfitting.

\subsection{Further Results on the \texttt{baseline} dataset}

As described in Section \ref{Sec:reproduction_bias}, the empirical results show that the \texttt{baseline} dataset is a much more easier dataset compared to the \texttt{not\_included} dataset with naive inference. Since the majority of expressions in the \texttt{baseline} dataset are expressions that were included in the training dataset, by applying test-time strategies to this dataset, we can test whether the test-time strategies help in reproducing the expressions from the training data.

In Figure \ref{Fig:vs_baseline_normal} and Figure \ref{Fig:vs_baseline_nsrwh}, we present the results concerning the numerical accuracy for various test-time strategies. We also tested with NSRwH as well as the test-time strategies described in Section \ref{Sec:test_time_computation}.
The results show that only beam search with a large beam size significantly improves performance, aligning with the previous result in Section \ref{Sec:test_time_computation} that providing the model with additional information often prevents the model from reproducing expressions from the training data.

\begin{figure}[htbp]
\begin{center}
\centerline{\includegraphics[width=0.5\columnwidth]{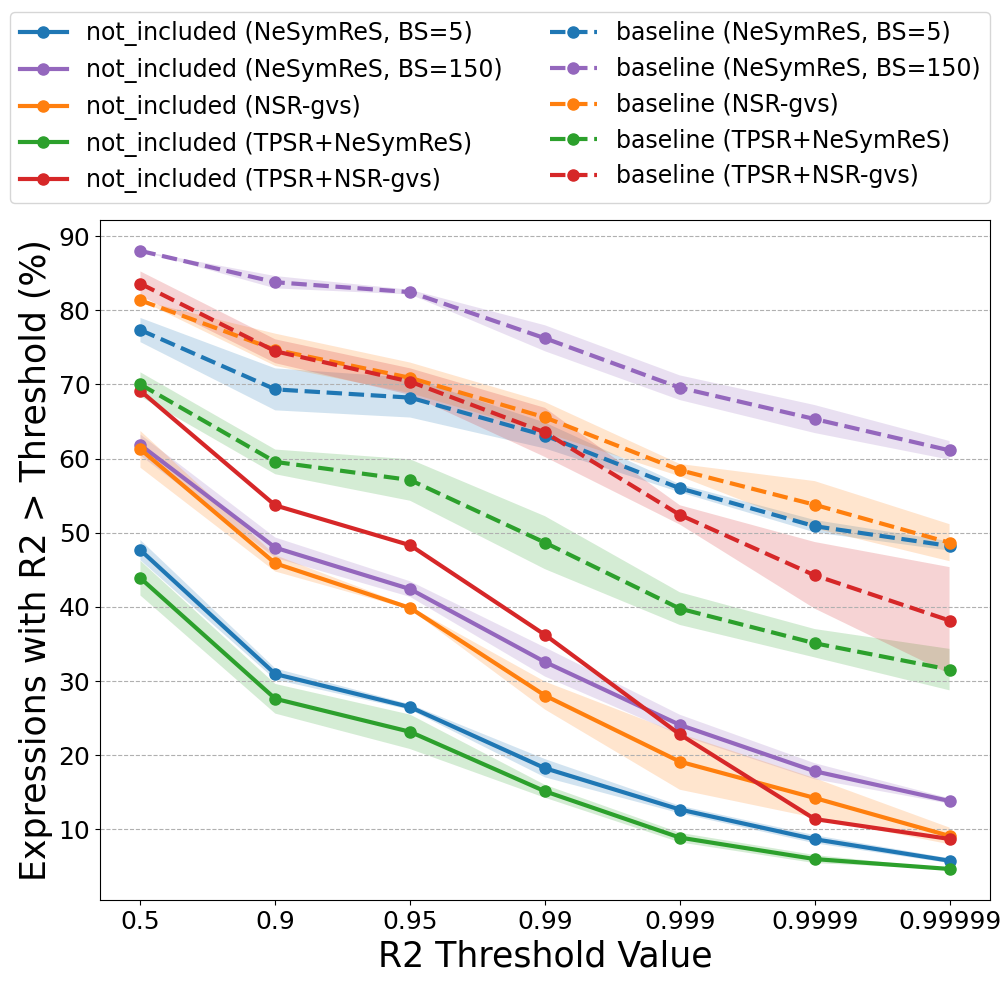}}
\caption{Numerical accuracy for various test-time strategies on the \texttt{baseline} (and \texttt{not\_included}) dataset. The y-axis shows the proportion of expressions that exceeded the $R^2$ thresholds of $0.5$, $0.9$, $0.95$, $0.99$, $0.999$, $0.9999$ and $0.99999$, respectively. Only beam search with a large beam size significantly improves performance for the \texttt{baseline} dataset.}
  \label{Fig:vs_baseline_normal}
\end{center}
\end{figure}

\begin{figure}[htbp]
\begin{center}
\centerline{\includegraphics[width=0.5\columnwidth]{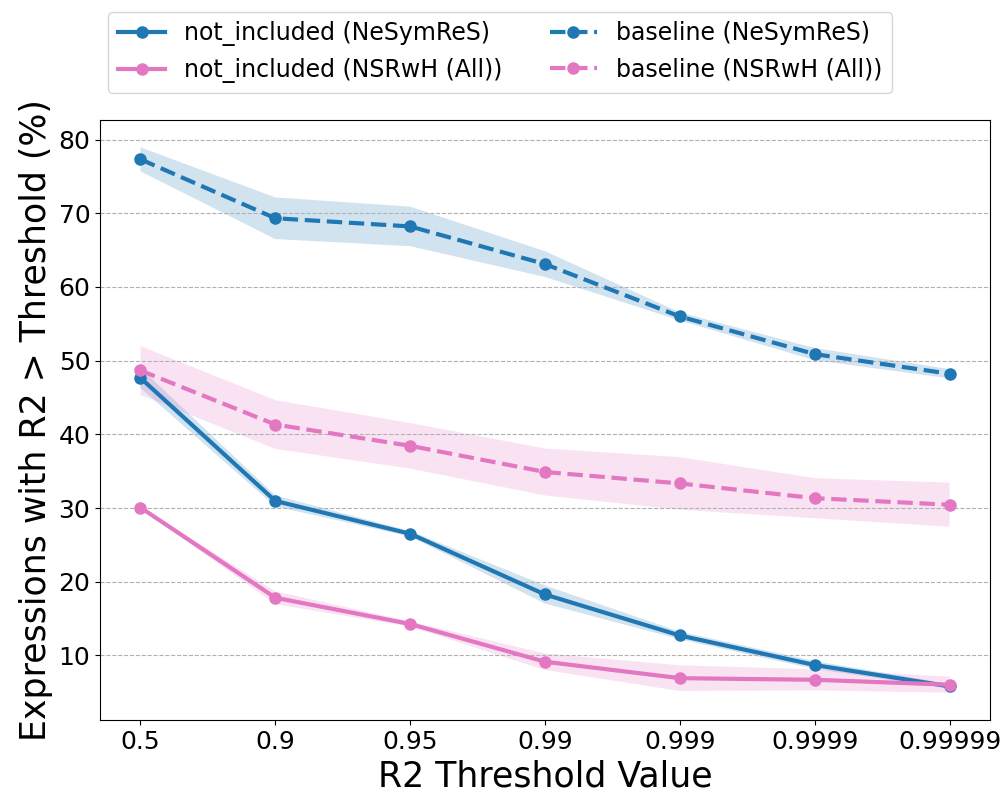}}
\caption{umerical accuracy for NSRwH on the \texttt{baseline} (and \texttt{not\_included}) dataset. The y-axis shows the proportion of expressions that exceeded the $R^2$ thresholds of $0.5$, $0.9$, $0.95$, $0.99$, $0.999$, $0.9999$ and $0.99999$, respectively. Adding additional information significantly curves the performance in the \texttt{baseline} dataset.}
  \label{Fig:vs_baseline_nsrwh}
\end{center}
\end{figure}

\newpage

\section{Theoretical Background and Proof} \label{Sec:theory}

\begin{figure}[htbp]
\begin{center}
\centerline{\includegraphics[width=\columnwidth]{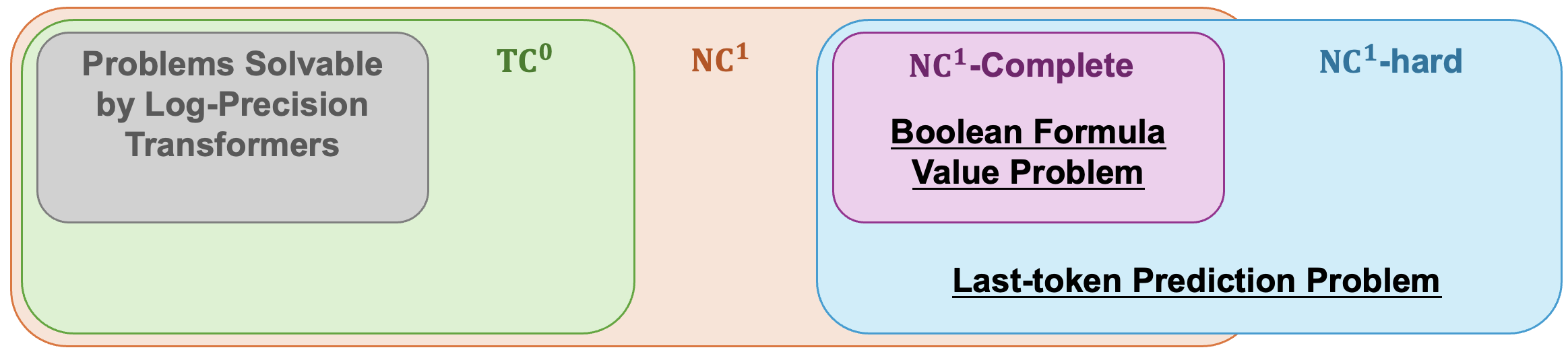}}
\caption{Overview of the theoretical analysis. When assuming $\mathsf{TC}^0 \neq \mathsf{NC}^1$, the last-token prediction problem is $\mathsf{NC}^1$-hard (due to reduction from Boolean formula value problem) while Transformers can only solve problems in $\mathsf{TC}^0$.}
\label{Fig:theory_overview}
\end{center}
\end{figure}

%%%%%%%%%%%%%%%%%%%%%%%%%%%%%%%%%%%%%%%%%%%%%%%%%%%%%%%%%%%%%%%%%%%%%%%%
%  Theorem : Bounded-depth log-precision Transformers fail
%            Last-Token Prediction
%%%%%%%%%%%%%%%%%%%%%%%%%%%%%%%%%%%%%%%%%%%%%%%%%%%%%%%%%%%%%%%%%%%%%%%%

In this section, we provide background knowledge, detailed settings, and a complete proof for the theoretical result presented in Section \ref{Sec:theoretical_analysis}.

%-------------------------------------------------
\subsection{Preliminary}\label{Ssec:theory_preliminary}
We first provide a brief overview of relevant circuit complexity classes.
We then define the class of log-precision Transformers and introduce its simulation guarantees. We also present a formal definition of the Boolean formula value problem, which we use in our proof.

\subsubsection{Circuit Complexity Classes}
We offer an explanation to several fundamental circuit complexity classes that are used in our theoretical analysis. Particularly, we discuss the complexity classes $\mathsf{AC}^0, \mathsf{TC}^0$ and , $\mathsf{NC}^1$. The relationship between these three classes can be summarized as follows:
\[
\mathsf{AC}^0 \subsetneq \mathsf{TC}^0 \subset \mathsf{NC}^1.
\]
Whether $\mathsf{TC}^0$ is a proper subset of $\mathsf{NC}^1$ is an open question, but it is widely believed that this is the case. For a more detailed and comprehensive introduction, we recommend reference to \citep{arora2009computational}.

\paragraph{Circuit class $\boldsymbol{\mathsf{AC}^{0}}\,$.}
The class $\mathsf{AC}^{0}$ consists of Boolean circuits of \emph{constant depth} and \emph{polynomial size} whose gates have unbounded fan‐in and are restricted to the basis
$\{\mathsf{AND},\mathsf{OR},\mathsf{NOT}\}$.
Intuitively, $\mathsf{AC}^{0}$ captures extremely shallow parallel
computation.

\paragraph{Circuit class $\boldsymbol{\mathsf{TC}^{0}}\,$.}
The class $\mathsf{TC}^{0}$ is an extension of $\mathsf{AC}^{0}$, where a gate called the \emph{majority gate} can be additionally used. A majority gate has unbounded fan-in and outputs false when half or more of the inputs are false, and true otherwise. Other definitions are the same as $\mathsf{AC}^{0}$. 

\paragraph{Circuit class $\boldsymbol{\mathsf{NC}^{1}}\,$.}
Circuits in $\mathsf{NC}^{1}$ are polynomial sized with the depth \emph{logarithmic} to the input size. They comprise of $\{\mathsf{AND},\mathsf{OR},\mathsf{NOT}\}$ gates with constant fan-in. The class $\mathsf{NC}^{1}$ contains several well-known problems such as the parity check on a bit string.

\subsubsection{Log-precision Transformers}
We assume bounded-depth log-precision Transformers throughout the theoretical analysis.
We first model the parametrized Transformer $\mathsf{TF}_{\boldsymbol \theta}$ as a next-token prediction function;
\begin{align}
  \mathsf{TF}_{\boldsymbol{\theta}} : \Gamma^{m}\times\mathbb R^{(d+1)\times n}
  \;\longrightarrow\;
  \Gamma,
\end{align}
i.e.\ the Transformer receives a length-$m$ prefix along with a numerical dataset~$\mathcal D$ and outputs a single token $u\in\Gamma$.

\begin{definition}[$(D,d)$-bounded log-precision Transformer]
\label{Def:boundedTransformer}
Let $k$ be the input length.
A $(D,d)$-bounded log-precision Transformer
is an encoder–decoder model that satisfies
\begin{enumerate}
\item constant depth $D=O(1)$,
\item hidden size $d\!\le\! Q(k)$ for a fixed polynomial $Q$,
\item the values at all layers, as well as the outputs of all key intermediate operations in it (attention, activation, arithmetic operators, etc.), are represented using $O(\log k)$ bits. 
\end{enumerate}
\end{definition}
For specific definitions of operations that enable approximation in $O(\log k)$ bits, please refer to Section $4$ and Appendix A of \citet{merrill2023parallelism}.
We introduce the simulation guarantees for bounded-depth log-precision Transformers as follows.

\begin{lemma}[Circuit simulation {\citep[Cor.\,2.1]{merrill2023parallelism}}]
\label{Lem:tc0-sim}
Any $(D,d)$-bounded log-precision Transformer can be simulated by a
family of $\mathsf{TC}^{0}$ circuits of size $\poly(k)$ and constant depth with respect to $k$.
\end{lemma}

\subsubsection{The Boolean Formula Value Problem}
Following the definition by \citet{buss1987boolean}, we introduce the definition of the Boolean formula value problem as follows.

\begin{definition}[Boolean formula value problem]
\label{Def:booleanformulavalue}
Let $\Lambda = \{0,1,\land,\lor,\lnot,(,)\}$ be the alphabet. A Boolean formula is a string defined recursively as follows:
\begin{enumerate}
\item $0$ and $1$ are Boolean formulae;
\item If $\mathbf{t}_1$ and $\mathbf{t}_2$ are two Boolean formulae, then $(\lnot \mathbf{t}_1), (\mathbf{t}_1 \land \mathbf{t}_2), (\mathbf{t}_1 \lor \mathbf{t}_2)$ are also Boolean formulae.
\end{enumerate}
When given a boolean formula $\mathbf{t}$, the goal of the Boolean formula value problem is to compute whether the evaluation result $\operatorname{eval}(\mathbf{t})$ of a given Boolean formula is $0$ or $1$.
\end{definition}

%-------------------------------------------------
%----------------------------------------------------------------------
\subsection{Main Theorem }
%----------------------------------------------------------------------
Prior to proving the main theorem, we state the following Lemma from \citet{feng2023towards}. The detailed proof of this Lemma can be found in the same paper. The Lemma states that $\mathsf{TC}^0$ circuits are capable of identifying the indexes of paired brackets in a string.

\begin{lemma}[Bracket parsing {\citep[Lem.\,D.3]{feng2023towards}}]
\label{Lem:brackets}
Consider any string $\mathbf{t}=t_1t_2\cdots t_n$ of length $n$ containing brackets `(', `)', and other characters, and all brackets in $\mathbf{t}$ are paired. Let $\boldsymbol{g}$ be a boolean function taking $\mathbf{t}$ as input and output $n$ pairs of integers defined as follows:
\begin{equation*}
\boldsymbol{g}_i(\mathbf{t})=\begin{cases}
(-1,j) & \text{if $t_i$ is a left bracket and $t_i,t_j$ are paired.}\\
(j,-1) & \text{if $t_i$ is a right bracket and $t_i,t_j$ are paired.}\\
(j,k) & \text{if $t_i$ is not a bracket, and $t_j,t_k$ are the nearest paired brackets containing $t_i$.}
\end{cases}
\end{equation*}
Then $\boldsymbol{g}$ can be implemented by the $\mathsf{TC}^0$ circuits.
\end{lemma}

We now proceed to prove the main theorem of our theoretical analysis.

\begin{theorem}[Bounded log-precision Transformer lower bound]
\label{Thm:impossibility}
Assume $\mathsf{TC}^0 \neq \mathsf{NC}^1$. For any integer $D$ and any polynomial $Q$, there exists a problem size $m$ such that  no $(D,d)$-bounded log-precision Transformer with $d\le Q(m)$ can solve $\operatorname{LastTokenPrediction}(m)$.
\end{theorem}

\begin{proof}
Fix $D$ and $Q$ and suppose, for contradiction, that for some
sufficiently large $m$ there exists an $(D,d)$-bounded
log-precision Transformer $\mathsf{TF}_\theta$ with $d\le Q(m)$ that
solves the problem of \(\operatorname{LastTokenPrediction}(m)\).

\paragraph{Step 1 (simulation).}\;
By Lemma~\ref{Lem:tc0-sim},
$\mathsf{TF}_\theta$ can be simulated by a
\(\mathsf{TC}^{0}\) circuit family of size $\poly(m)$.
Hence, under our assumption,  
\(\operatorname{LastTokenPrediction}(m)\in\mathsf{TC}^{0}\).

\paragraph{Step 2 ($\boldsymbol{\mathsf{TC}^{0}}$ construction).}\;
We show that there exists a $\mathsf{TC}^{0}$ circuit that can translate any instance of a Boolean formula value problem to an instance of the last-token prediction problem.

Let $\mathbf{t}$ be a boolean formula. There exists a $\mathsf{TC}^{0}$ circuit that performs:

\begin{enumerate}
\item \emph{Translation of $\,\mathbf{t}$ to a $\lor$-free Boolean formula $\,\mathbf{t^\prime}$}.
\item \emph{Conversion of $\,\mathbf{t^\prime}$ to its prefix notation $\,\mathbf{t^\prime_\mathrm{pre}}$}.
\item \emph{Conversion of $\,\mathbf{t^\prime_\mathrm{pre}}$ to a token sequence $\mathbf{s} \in \Gamma^\ast$ by the following procedure:}
    \begin{enumerate}
    \item replace $0$ with $\{\times, x_1, x_2\}$ and $1$ with $\{-, x_1, x_2\}$;
    \item replace $\land$ with $\times$;
    \item replace $\lnot$ with $\{-, -, x_1, x_2\}$.
    \end{enumerate}
\item \emph{Local edits:}
    \begin{enumerate}
    \item prepend $+$ to $\mathbf{s}$ to form the incomplete token sequence $\widetilde{\mathbf{s}}$;
    \item set $n = 2$, and attach the data points $(x_{1,1},x_{2,1},y_1)=(1,0,1)$ and $(x_{1,2},x_{2,2},y_2)=(0,-1,0)$ to the input numerical data $\mathcal{D}$;
    \item define the metric as the mean squared error: 
        $\mathcal L(\mathbf{y},\hat{\mathbf{y}}) = \frac{1}{n}\sum_{i=1}^n (y_i - \hat{y_i})^2$
    \end{enumerate}
\end{enumerate}
To perform the first step, for all $\lor$ in $\mathbf{t}$, we must replace the nearest left bracket containing $\lor$ with $\lnot(\lnot$ and also replace $\lor$ with $\land\lnot$. By using the results of Lemma \ref{Lem:brackets}, it follows that this operation can be performed by a circuit within $\mathsf{TC}^{0}$ complexity. The second step can be implemented by $\mathsf{AC}^{0}$ circuits, according to \citet[Cor. 11]{buss1987boolean}. Since the third and fourth steps only involve replacing and extending obtained sequences, these steps can also be implemented by $\mathsf{AC}^{0}$ circuits.

\paragraph{Step 3 (soundness of the reduction).}\;
When $\operatorname{eval}(\mathbf{t}) = 0$, the losses $\mathcal{L}(\mathbf{y},e_{(\widetilde{\mathbf{s}},u)}(\mathbf{x}))$ for each leaf token $u \in \{x_1, x_2, C\}$ can be computed as follows: 
\begin{align*}
\begin{dcases}
\mathcal{L}(\mathbf{y},e_{(\widetilde{\mathbf{s}},x_1)}(\mathbf{x}))
= \dfrac{1}{2} \sum_{i=1}^2 (y_i - x_{1,i})^2
= \dfrac{1}{2}(0^2 + 0^2) = 0, \\
\mathcal{L}(\mathbf{y},e_{(\widetilde{\mathbf{s}},x_2)}(\mathbf{x})) 
= \dfrac{1}{2} \sum_{i=1}^2 (y_i - x_{2,i})^2
= \dfrac{1}{2}(1^2 + 1^2) = 1,\\
\mathcal{L}(\mathbf{y},e_{(\widetilde{\mathbf{s}},C)}(\mathbf{x})) 
= \underset{c \in \mathcal{C}}{\operatorname{arg min}}~\dfrac{1}{2} \sum_{i=1}^2 (y_i - c)^2
= \underset{c \in \mathcal{C}}{\operatorname{arg min}}~\dfrac{1}{2}((1 - c)^2 + (- c)^2) \geq \dfrac{1}{4},
\end{dcases}
\end{align*}
where $\mathcal{C}$ is the interval from which numeric constants are drawn, and $e_{\mathbf{s}} = \mathrm{expr}(\mathbf{s},\mathcal{D})$ is the mapping function defined in Section \ref{Sec:theoretical_analysis}. When $\operatorname{eval}(\mathbf{t}) = 1$, the losses $\mathcal{L}(\mathbf{y},e_{(\widetilde{\mathbf{s}},u)}(\mathbf{x}))$ can be computed as follows: 
\begin{align*}
\begin{dcases}
\mathcal{L}(\mathbf{y},e_{(\widetilde{\mathbf{s}},x_1)}(\mathbf{x}))
= \dfrac{1}{2} \sum_{i=1}^2 (y_i - (1 + x_{1,i}))^2
= \dfrac{1}{2}((-1)^2 + (-1)^2) = 1, \\
\mathcal{L}(\mathbf{y},e_{(\widetilde{\mathbf{s}},x_2)}(\mathbf{x})) 
= \dfrac{1}{2} \sum_{i=1}^2 (y_i - (1 + x_{2,i}))^2
= \dfrac{1}{2}(0^2 + 0^2) = 0,\\
\mathcal{L}(\mathbf{y},e_{(\widetilde{\mathbf{s}},C)}(\mathbf{x})) 
= \underset{c \in \mathcal{C}}{\operatorname{arg min}}~\dfrac{1}{2} \sum_{i=1}^2 (y_i - (1 + c))^2
= \underset{c \in \mathcal{C}}{\operatorname{arg min}}~\dfrac{1}{2}((-c)^2 + (- 1 - c)^2) \geq \dfrac{1}{4},
\end{dcases}
\end{align*}

Consequently, when $\operatorname{eval}(\mathbf{t}) = 0$, the result for the corresponding last-token prediction problem is $u^\ast = x_1$, while when $\operatorname{eval}(\mathbf{t}) = 1$, the result is $u^\ast = x_2$.
Hence the mapping introduced in Step 2 is a valid $\mathsf{TC}^{0}$ many-one reduction from the Boolean formula value problem to the last-token prediction problem.

\paragraph{Step 4 (contradiction).}  
The Boolean formula value problem is \(\mathsf{NC}^{1}\)-complete under
\(\mathsf{AC}^{0}\) reductions
\citep[Thm. 9]{buss1987boolean}.
Hence Step 2 and Step 3 indicate that \(\operatorname{LastTokenPrediction}(m)\notin\mathsf{TC}^{0}\),
contradicting Step 1 and the assumed strict inclusion
\(\mathsf{TC}^{0}\subsetneq\mathsf{NC}^{1}\).
Therefore, such a Transformer cannot exist.
\end{proof}

%=============================================================%
%  PAC approximation via iterated self-verification          %
%=============================================================%
\subsection{PAC approximation via iterated self-verification}
We further present theoretical analysis regarding the performance of the proposed method, NSR-gvs.

\begin{assumption}\label{Assumption:iterated_self-verification}
We make the following assumptions.
\begin{enumerate}
\item \textbf{Hypothesis class}.  
      Fix a maximum depth $D_{0}$ and a grid spaced in $\varepsilon/2$ on $[-1,1]$.  
      \[
         \mathcal U:=\{e:\mathrm{depth}(e)\le D_{0}\},\quad
         U:=|\mathcal U|\le \poly(n),
      \]
      where $n:=|\mathcal D|$.

\item \textbf{Data}.  
      $\mathcal D=\{(x_i,y_i)\}_{i=1}^{n}$ with $y_i\in[-1,1]$ and  
      \(
        e^{\star}=\arg\min_{e\in\mathcal U}\mathrm{MSE}(e;\mathcal D).
      \)

\item \textbf{Transformer}.  
      A depth-$L$ log-precision Transformer $T$ ($L$ constant).

\item \textbf{Exact oracle}.  
      A routine $\mathcal M$ returns $\mathrm{MSE}(e;\mathcal D)$ for any $e$.

\item \textbf{Hit rate}.  
      If every subtree of $e^{\star}$ is present in the prompt,
      $T$ outputs $e^{\star}$ with probability at least $\beta\in(0,1]$.

\item \textbf{Dictionary growth}.  
      Each round appends at least one \emph{uniformly random}
unseen subtree to the prompt (chosen without replacement; if fewer than $r$ remain, insert all).
\end{enumerate}
\end{assumption}

%---------------------  Theorem  ----------------------------
\begin{theorem}[informal]
Let the algorithm cycle long enough for its prompt to have seen every possible sub-expression; then keep running a few more rounds.
With very high probability, it returns a formula whose error is no worse than an optimally chosen tree by more than a tiny tolerance, and it has queried the oracle only a moderate, logarithmically growing number of times.
\end{theorem}
\begin{theorem}[PAC guarantee]
\label{Thm:selfverify-corrected}
Run the loop
\[
   e_t \gets T(\text{prompt});\quad
   R_t \gets \mathcal M(e_t);\quad
   \text{prompt}\mathrel{+}=\text{sub-trees}(e_t)
\]
for a \emph{burn-in}
\(
   B=\bigl\lceil\tfrac{U}{r}\ln\!\bigl(2^{D_{0}}/(\delta/2)\bigr)\bigr\rceil
\)
rounds, followed by
\(
   R=\bigl\lceil\tfrac{\ln(2/\delta)}{\beta}\bigr\rceil
\)
additional rounds, and return the best-so-far expression
$e_{\mathrm{best}}$.

Then, Under Assumption \ref{Assumption:iterated_self-verification}, for any $\varepsilon,\delta\in(0,1)$,
\[
  \Pr\!\Bigl[
       \mathrm{MSE}\bigl(e_{\mathrm{best}},\mathcal D\bigr)
       \le \mathrm{MSE}\bigl(e^{\star},\mathcal D\bigr)+\varepsilon
     \Bigr]\;\ge\;1-\delta,
\qquad
  \#\text{oracle calls}\;=\;\mathcal O\!\bigl(U\ln(1/\delta)\bigr).
\]
\end{theorem}

%-------------------------- Proof -----------------------------%
\begin{proof}
\textbf{(i) Burn-in.}
There are $K\!\le\!2^{D_{0}}$ distinct sub-trees of $e^{\star}$.  
Drawing $r\ge1$ uniform sub-trees per round, the probability a fixed
sub-tree is never drawn in $B$ rounds is
$(1-\tfrac{r}{U})^{B}\le e^{-rB/U}\le\delta/(2K)$.
A union bound over all $K$ sub-trees implies that, after $B$ rounds, the
prompt contains \emph{every} sub-tree of $e^{\star}$ with probability
at least $1-\delta/2$.
  
\smallskip
\textbf{(ii) Post burn-in success.}
Condition on the burn-in success event.  
By assumption (3) each subsequent round now hits $e^{\star}$ with
probability at least $\beta$, regardless of possible prompt changes.  
Therefore
\[
   \Pr[\text{miss in all }R\text{ rounds}]
     \;\le\;(1-\beta)^{R}
     \;\le\;e^{-\beta R}
     \;\le\;\delta/2,
\]
for $R=\lceil\ln(2/\delta)/\beta\rceil$.

\smallskip
\textbf{(iii) Union bound.}
Total failure probability $\le\delta/2+\delta/2=\delta$.

\smallskip
\textbf{(iv) Quality of $e_{\mathrm{best}}$.}
Whenever $e^{\star}$ appears, the exact oracle certifies its MSE; the
algorithm stores it permanently.  Hence on the complement of failure
the returned expression meets the stated error bound.

\smallskip
\textbf{(v) Oracle calls.}
At most one full-expression evaluation per round, so the algorithm
issues $B+R=\mathcal O(U\ln(1/\delta))$ oracle calls.

\end{proof}

\end{document}